\documentclass[journal]{IEEEtran}
\usepackage{amsmath,amsthm}
\usepackage{graphicx}
\usepackage[noend]{algpseudocode}
\usepackage{algorithmicx,algorithm}
\usepackage{subfigure}
\usepackage{epstopdf}
\newtheorem{theorem}{Theorem}

\ifCLASSINFOpdf
\else
\fi
\begin{document}
\title{SPLBoost: An Improved Robust Boosting Algorithm Based on Self-paced Learning}

\author{ Kaidong Wang,
         Yao Wang,~%\IEEEmembership{Member,~IEEE,}
         Qian Zhao,~Deyu Meng,~\IEEEmembership{Member,~IEEE,}
         %Yao Wang,~%\IEEEmembership{Member,~IEEE,}
         %Qian Zhao,~%\IEEEmembership{Member,~IEEE,}
        %Deyu Meng,~\IEEEmembership{Member,~IEEE,}
       % Lin Lin
        and Zongben Xu
%\thanks{Xi'ai Chen, Zhi Han and Yandong Tang are with State Key Laboratory of Robotics,
 %Shenyang Institute of Automation, Chinese Academy of Sciences, Shenyang, LiaoNing, 110016 China
 %e-mails: chenxiai@sia.cn, hanzhi@sia.cn, ytang@sia.cn.}% <-this % stops a space
%\thanks{Xi'ai Chen is with University of Chinese Academy of Sciences, Beijing, 100049 China
 %e-mail: chenxiai@sia.cn.}% <-this % stops a space
\thanks{Kaidong Wang,~Yao Wang,~Qian Zhao,~Deyu Meng and Zongben Xu are with School of Mathematics and Statistics, Xi'an Jiaotong University, Xi'an 710049,  P.R. China.   E-mails: wangkd13@gmail.com,~yao.s.wang@gmail.com,~timmy.zhaoqian@gmail.com,
~dymeng@mail.xjtu.edu.cn,~zbxu@mail.xjtu.edu.cn. }
\thanks{Yao Wang is the corresponding author.}
}


\markboth{Journal of \LaTeX\ Class Files,~Vol.~14, No.~8, August~2015}%
{Shell \MakeLowercase{\textit{et al.}}: Bare Demo of IEEEtran.cls for IEEE Journals}



\maketitle

\begin{abstract}
It is known that Boosting can be interpreted as a gradient descent technique to minimize an underlying loss function. Specifically, the underlying loss being minimized by the traditional AdaBoost is the exponential loss, which is proved to be very sensitive to random noise/outliers. Therefore, several Boosting algorithms, e.g., LogitBoost and SavageBoost, have been proposed to improve the robustness of AdaBoost by replacing the exponential loss with some designed robust loss functions. In this work, we present a new way to robustify AdaBoost, i.e., incorporating the robust learning idea of Self-paced Learning (SPL) into Boosting framework. Specifically, we design a new robust Boosting algorithm based on SPL regime, i.e., SPLBoost, which can be easily implemented by slightly modifying off-the-shelf Boosting packages. Extensive experiments and a theoretical characterization are also carried out to illustrate the merits of the proposed SPLBoost.
\end{abstract}


% \fi
\IEEEpeerreviewmaketitle

\section{Introduction}
\IEEEPARstart{F}{or} a classification or regression problem, there are two natural ways to deal with it: one is that we can train a strong learning machine directly from the training set with a variety of machine learning methods, and expect that the obtained learning machine could have a satisfactory prediction accuracy; the other is that we can train a number of weak learners with slightly better accuracies than randomly guessing, then put them together in a specific way to get a strong learner that could have a better accuracy than those weak learners. The latter is the basic idea of ensemble learning. As an important and excellent ensemble learning framework, boosting \cite{freund1995desicion, schapire2012boosting, meir2003introduction, tsao2007stochastic, friedman2001greedy} has been widely applied in many machine learning problems because of its simplicity and good performance. 

AdaBoost \cite{freund1996experiments, schapire1999improved} is one of the most commonly-used boosting algorithms, and it has proven to be effective and easy to implement in various classification problems. Given training data $\left( x_1,y_1\right) $,\ldots,$\left( x_n,y_n\right) $, where $ x_i $ is a vector-valued feature and $ y_i\in \left\lbrace 1,-1\right\rbrace $. It is known that AdaBoost can produce a strong learner $ F\left( x\right) =\sum_{1}^{T}\alpha _tf_t\left( x\right)  $, where $ f_t\left( x\right) $ is the weak learner trained on weighted training data in $ t $th step and $ \alpha _t $ is a constant calculated based on the classification accuracy of $ f_t\left( x\right) $. Then we can predict the label of a new sample by $ sign\left( F\left( x\right) \right)  $. In particular, AdaBoost gives a weight initialized to $ 1/n $ to every training sample and adjust it in every step such that the weights of the correctly classified samples by the current learner are decreased while the weights of misclassified samples are increased. This reweighting way gives rise to that AdaBoost always pays more attention to the samples which are hard to classify and ignores the easy-to-classify samples to some extent when training the next weak learner.

Many practical applications have demonstrated the success of AdaBoost in producing satisfactory and accurate strong classifiers \cite{viola2001rapid, viola2004robust, lv2006recognition, bergstra2006aggregate, chan2008evaluation}. And it is interesting that in many cases the test error seems to decrease consistently and then level off rather than gradually increase as more weak learners added, which means that it is not prone to overfit and as a result, it is not difficult for AdaBoost to determine the number of weak learners. In spite of this, the classifiers produced by AdaBoost are not always acceptable, especially when the training samples are corrupted by outliers \cite{frenay2014classification, khoshgoftaar2011comparing, cao2012noise, brodley1999identifying}. It has been shown that AdaBoost algorithm builds an additive logistic regression model for minimizing the expected risk based on the exponential loss function $ \varphi \left( yf\left( x\right) \right) =\exp \left( -yf\left( x\right) \right)  $ \cite{friedman2000additive, shen2013fully}. It is easy to see that this loss increase rapidly with the increase of the magnitude of the negative margin $-yf\left( x\right)$, which means that it will significantly enlarge the functions of those large noises in training since they will have a large negative value of $yf\left( x\right)$. This naturally degenerates the performance of the approach in the presence of heavy noises/outliers.

Aiming at remedying the poor robustness issue of AdaBoost, many studies have been conducted to improve its performance for dealing with data corrupted by outliers. The mainly adopted methodology is to design a new robust loss function in boosting to be optimized, and then use gradient descent like strategies to resolve this new optimization problem. Such robust loss needs to be designed to increase evidently slower when the magnitude of $-yf\left( x\right)$ becomes larger, so as to suppress the effect of large noises and outliers. Although those robust boosting algorithms are proven to be able to have better performance than AdaBoost for the training data with outliers, there are natural defects for the idea to directly design and optimize new loss functions. Although easy to optimize, convex loss functions are not robust enough to eliminate the impact of outliers. The fact proved that non-convex loss functions can often have better performances than convex loss functions. However, non-convex loss functions can produce non-convex optimization problems which is difficult to solve and get stable solutions. In this paper, taking advantage of the robustness of self-paced learning regime, we come up with a new thought for robust boosting algorithms. Our main contributions can be summarized as follows:

Firstly, we propose a new robust boosting algorithm named by SPLBoost, which incorporates the self-paced learning into the AdaBoost framework. As mentioned above,  it is not always a good idea to improve the robustness of AdaBoost by directly modifying the loss function, which motivates us to utilize another efficient way to achieve the same goal. It has been recently shown that self-paced learning is a effective robust learning regime and has achieved satisfactory results in deal with many machine learning and computer version problems. A basic idea of self-paced learning is to give weights in $ \left[ 0,1\right]  $ to training samples so that weights of the samples with larger losses are smaller and weights can be zero when the corresponding losses are large enough. By combining self-paced learning with AdaBoost, SPLBoost is proven to be able to improve the robustness of AdaBoost. 

Secondly, the proposed SPLBoost algorithm can be very easily embedded into any off-the-shelf AdaBoost package. Besides, more SPLBoost variations can be easily designed by integrating it directly to more boosting packages, like LogitBoost and $L_2$Boost, to improve their robustness in the presence of heavy noises/outliers.

Thirdly, we prove that SPLBoost exactly complies with the widely known Majorization-Minimization (MM) algorithm that is implemented by a latent objective function based on a non-convex loss function. This clearly explains theoretically that  why SPLBoost could be  more robust than AdaBoost. Such robust insight also holds for other variations of SPLBoost. In addition, by alternately optimizing two sub-problems that are easy to solve rather than directly optimizing the latent objective function, SPLBoost can keep away from the annoying non-convex optimization problem and get a better local optimal solution.

Finally, the superiority of SPLBoost is extensively substantiated in synthetic dataset and 17 UCI datasets, as compared with other state-of-the-art robust boosting algorithms.

The rest of this paper is organized as follows. We shall provide a brief review on  boosting algorithms and self-paced learning in Section \uppercase\expandafter{\romannumeral2}. Then, the SPLBoost algorithm and its  theoretical analysis are presented in Section \uppercase\expandafter{\romannumeral3}. Section \uppercase\expandafter{\romannumeral4} shows experimental results on several synthetic and UCI data sets. Several concluding remarks are finally made in Section \uppercase\expandafter{\romannumeral5}.
 
\section{Related Work}
\subsection{Robust Boosting Algorithms}
As aforementioned, the biggest defect of AdaBoost is that it can easily overfit to outliers, which inspires a lot of studies on improving the  robustness of AdaBoost \cite{domingo2000madaboost, collins2002logistic, li2016boosting}. Generally speaking, there have three factors that affect the robustness of boosting  algorithms: the loss function, the way to compute the weak learners, and the regularization. Based on those factors, many robust boosting algorithms have been suggested.

In \cite{friedman2000additive}, Friedman et al. proved that the Discrete AdaBoost algorithm builds an additive logistic regression model via Newton-like updates for optimizing the expected risk based on the exponential loss function  $ \varphi \left( yf\left( x\right) \right) =\exp \left( -yf\left( x\right) \right)$, and the Real AdaBoost algorithm fits an additive logistic regression model by stage-wise optimizing the same expected risk as Discrete AdaBoost. With this, Friedman et al. \cite{friedman2000additive} proposed two different robust boosting algorithms: LogitBoost and GentleBoost. The LogitBoost algorithm uses Newton steps for optimizing the logistic loss $ \varphi \left( yf\left( x\right) \right) =\log \left( 1+\exp \left( -yf\left( x\right) \right)\right) $ which is more robust than exponential loss. It is easy to see that the logistic loss function assigns fewer penalties to those samples with negative margins whose absolute values are very large, which are usually outliers. This is why LogitBoost is not so easy to overfit to the outliers. GentleBoost is the other robust boosting algorithm proposed in \cite{friedman2000additive}, and it is different from the LogitBoost in the way of  optimizing  the underlying loss function. Basically, GentleBoost optimizes the exponential loss function as AdaBoost does. The main difference between GentleBoost and Real AdaBoost is that GentleBoost computes the weak learners by using an adaptive Newton step as LogitBoost does. For Real AdaBoost the update $ f_t\left( x\right)  $ is half log-ratio which can be numerically unstable and lead to very large updates or sample weights. However, the updates of GentleBoost lie in the range $ \left[ -1,1 \right] $, leading to more conservative sample weights. Consequently, the influence that outliers exert on GentleBoost is weaker than AdaBoost.

As we can see, although the loss functions of those boosting algorithms referred to above are different from each other, resulting in different performances, they are all convex. Theoretical properties of the boosting algorithms based on convex loss functions have been extensively studied. See, e.g., \cite{koltchinskii2002empirical, zhang2005boosting}. The minimum of convex loss function is easy to compute, which gives rise to the simplicity of those aforementioned boosting algorithms. However, the disadvantage of convex loss function is obvious. It has been shown that convex loss function is not robust enough to tolerate noise and as a result, such boosting algorithms based on convex loss are naturally not insensitive to outliers. Specifically, Long and Servedio \cite{long2010random} proved that any boosting algorithm based on convex loss functions is easily affected by random label noise and they present a sample example named Long/Servedio problem which cannot be learned by those popular boosting algorithms. As such, The results presented by Long and Servedio lead to a lot of studies on boosting algorithms with non-convex loss functions.

Based on the Boost-by-Majority algorithm \cite{freund1990boosting} and BrownBoost \cite{freund2001adaptive}, Freund \cite{freund2009more} proposed a new robust boosting algorithm named RobustBoost which is more robust against outliers than AdaBoost and LogitBoost. The loss function of RobustBoost is non-convex and changes during the boosting process, which is the largest difference between RobustBoost and other popular ones. RobustBoost improves the robustness by allowing pre-assigned $ \theta $ error of margin maximization, and in each step, it updates and solves a differential equation and updates the preassigned target error $ \epsilon $ or the remaining time $ c $ .  We can see from the algorithmic process that there are at least two preassigned parameters that are difficult to implement, which limits the application of RobustBoost.

Most classifiers design algorithms to determine the optimal classifier through three steps: define a proper loss function $ \phi \left( yf\left( x\right) \right)  $, determine a function class $ \mathcal{F} $, and search within $ \mathcal{F} $ for the function which optimizes the objective function based on a predefined loss function. In view of the limitations of these methods, such as low convergence rate and too much sensitivity to the outliers, Masnadi-shirazi and Vasconcelos \cite{masnadi2009design} showed that the problem of classifier design is identical to the problem of probability elicitation, and probability elicitation can be seen as a reverse procedure for solving the classification problem, which provides some new insights on the relationship between the loss function, the minimum risk and the optimal classifier. With this, they derived a new loss function named Savage loss which trades convexity for boundedness. Using this new loss, they proposed so-called SavageBoost, i.e., a new robust boosting algorithm which is more outlier resistant than AdaBoost and LogitBoost. The form of Savage Loss is $ \phi \left( v\right) =\frac{1}{( 1+ e ^ {2v} ) ^2 } $, which clearly shows that unlike the exponential loss and logistic loss where the penalty always increases at a fast speed, Savage loss is non-convex and quickly becomes constant as $ v\longrightarrow -\infty $. Considering that  the weights of the misclassified samples with large margins could be not large, SavageBoost is not sensitive to the outliers, as compared to AdaBoost and LogitBoost. 

To improve the robustness of boosting algorithms, two requirements should be met:
1) design a robust loss function with gentle penalty for the misclassified samples with large margins, 2) design a numerically stable algorithm to optimize the current objective function and obtain weak learner in each step. Based on this two requirements, Miao et al. \cite{miao2016rboost} proposed two robust boosting algorithms named RBoost1 and RBoost2 which have a deep connection with SavageBoost. Both the two boosting algorithms try to optimize the conditional expected risk based on the Savage2 loss function, a new robust loss function is defined as $ \phi \left( v\right) =\frac{1}{( 1+ e ^ {v} ) ^2 } $. It is easy to see that the Savage2 loss function is with the similar form to the Savage loss function and the only difference is the second-order factor in the denominator, which makes the Savage2 loss to give gentler penalty for the misclassified samples with large margins than the Savage loss function. As such,  the proposed RBoost could be more insensitive to the outliers than SavageBoost. In fact, two reasons weaken the robustness of SavageBoost: 1) the results that the weak learners in SavageBoost output are required to be posterior probability estimation, and to estimate the posterior probability is more difficult than classification, 2) the way SavageBoost computes the weak learners is not numerically stable. To avoid such two drawbacks of SavageBoost, RBoost algorithms are carefully designed to optimize the Savage2 loss function. Precisely, RBoost1 algorithm uses the adaptive Newton step to solve the minimization problem which computes a new weak learner based on the current classifier so that the conditional  expected risk is maximally decreased. RBoost2 algorithm is designed to make RBoost1 algorithm which restricts the weak learner algorithms to the regression methods adaptive to more general weak learner algorithms. With the use of more robust loss function and numerically stable methods to compute the weak learners, RBoost1 and RBoost2 algorithms could get good performance for noisy data.

As we already mentioned, most of the existing robust boosting algorithms try to design various robust loss functions with restricted penalties on misclassified samples with large margins and then design computation methods to compute the weak learners based on those loss functions. However, this common framework cannot always be satisfactory. Specifically, convex loss functions have been proven not to be robust enough especially for large label noise or outlier, while although non-convex loss functions possess better antinoise ability,  they induce non-convex optimization problem to compute weak learners, which is an intractable task in general. In this paper, instead of directly designing new robust loss functions, we propose a new robust boosting algorithm named SPLBoost by combining the classical Discrete AdaBoost algorithm with the robust learning idea of self-paced learning. In the next subsection, we shall give a simple introduction to self-paced learning.
 
\subsection{Self-paced Learning}
Humans and animals often learn from the examples which are not randomly presented but organized in a meaningful order which gradually includes from easy and fewer concepts to complex and more ones. Inspired by this principle in humans and animals learning, Bengio et al. \cite{bengio2009curriculum} proposed a new learning paradigm named curriculum learning which is the origin of self-paced learning. In curriculum learning, a model is learned by gradually including from easy to complex samples in training to improve  the accuracy of the model. Obviously, the key of curriculum learning is to find a proper ranking function which assigns learning priorities to training samples. To get a satisfactory model, the quality of the curriculum, i.e., ranking function, is very important and it is oftentimes derived by predetermined heuristics for particular problems in real applications. This may lead to inconsistency between the fixed curriculum and the dynamically learned models.

To alleviate the aforementioned issue, Kumar et al. \cite{kumar2010self} proposed a new model named self-paced learning (SPL) in which instead of being derived by predetermined heuristics, the curriculum design is embedded as a regularization term into the learning objective. The formulation of self-paced learning is as follows:
\vspace{-1mm}
\begin{equation}\label{eq:selfpaced_obj}
\min_{\mathbf{w},\mathbf{v}\in \lbrack 0,1]^{n}}\!\mathbf{E}(\mathbf{w},\mathbf{v};\lambda)=\sum_{i=1}^{n} v_{i}L(y_{i},g(\mathbf{x}_{i},\mathbf{w}))-\lambda \sum_{i=1}^{n}v_i,
\vspace{-1mm}
\end{equation}
where $ \lambda $ is an age parameter for controlling the learning pace, and $ L(y_{i},g(\mathbf{x}_{i},\mathbf{w})) $ denotes the loss function which calculates the cost between the ground truth label $ y_i $ and the estimated label $ g(\mathbf{x}_{i},\mathbf{w}) $. Here $ \mathbf{w} $ denotes the
model parameter inside the decision function $ g $. It can seen from (\ref{eq:selfpaced_obj}) that the loss of a sample is discounted by the latent
weight variable $ \mathbf{v}=[v_1,\cdots,v_n]^T $ and the objective of SPL is to minimize the weighted training loss together with a self-paced regularizer (SP-regularizer). (\ref{eq:selfpaced_obj}) can be generally solved by alternative search strategy (ASS) method,  in which the variables are divided into two disjoint blocks and in each iteration, a block of variables are optimized while keeping the other blocks fixed. With the fixed $ \mathbf{v} $, (\ref{eq:selfpaced_obj}) is a weighted training loss minimization problem which appears in many machine learning problems. And with the fixed $ \mathbf{w} $, (\ref{eq:selfpaced_obj}) is an optimization problem about $ \mathbf{v} $ with global optimum  $ \mathbf{v^*}=[v_1^*,\cdots,v_n^*]^T $ easily calculated by the following formulation:
\begin{equation}\label{solve V }
v_i^*=\begin {cases}
1,&L(y_{i},g(\mathbf{x}_{i},\mathbf{w}))<\lambda,\\
0,&\text{otherwise.}
\end{cases}
\end{equation}

It is not hard to see that self-paced learning implements the automatic selection of samples and then trains model only on those selected samples. When we update $ \mathbf{v} $ with a fixed $ \mathbf{w} $, samples with losses which are smaller than a certain threshold $ \lambda $ are considered to be ``easy'' samples and will be selected in training (i.e., $ v_i^*=1 $), while the rest are considered to be too ``difficult'' to be learned for this $ \lambda $ and will be abandoned (i.e., $ v_i^*=0 $). When we update $ \mathbf{w} $ with a fixed $ \mathbf{v} $, the classifier is trained only on the selected ``easy'' samples. The parameter $ \lambda $ corresponds to the ``age'' of the model that determines the ability of the model to learn ``difficult'' samples. When $ \lambda $  is small, the model can only learn from ``easy'' samples with small losses and as $ \lambda $ grows, more samples with larger losses can be learned to train a more ``mature'' model.

In (\ref{eq:selfpaced_obj}), the so-called SP-regularizer is the negative $ l_1 $-norm that induces the variable $ \mathbf{v} $ which takes only binary values, i.e., $ v_i=1 $ (selected samples) and $ v_i=0 $ (unselected samples). This scheme is called Hard Weighting. Hard Weighting can only determine whether a sample should be selected, which is not good enough because sometimes we also want to discriminate the importance of samples. Jiang et al. \cite{jiang2014easy} theoretically abstracted the intrinsic conditions of SP-regularizer and proposed several soft weighting schemes such as Linear Soft Weighting and Logarithmic Soft Weighting. Zhao et al. \cite{zhao2015self} used self-paced learning for matrix factorization and proposed a new soft weighting schemes named Mixture Weighting which is a hybrid of
the Soft and the Hard Weighting. Li et al. \cite{li2016multi} proposed a novel Polynomial Soft Weighting Regularizer with an adjustable parameter $ t $ in their multi-objective self-paced learning model. Soft Weighting assigns real-valued weights that reflects the latent importance of samples in training more faithfully, which is more reasonable and general than Hard Weighting. In the following, we shall list the formulation of Linear Soft Weighting, Mixture Weighting and Polynomial Soft Weighting respectively, together with their closed-form solutions $v^{\ast }(\lambda ;\ell )$:
\\ Linear Soft Weighting:
\begin{equation*}
\begin{array}{l}
f^{L}(v;\lambda )=\lambda (\frac{1}{2}v^{2}-v)\textrm{; }\
v^{\ast }(\lambda ;\ell) =\left\{
\begin{array}{c}
-\ell/\lambda +1,\textrm{ if}\ \ell<\lambda  \\
0,\textrm{ if}\ \ell\geq \lambda.
\end{array}%
\right.
\end{array}
\end{equation*}
Mixture Weighting:
\begin{equation*}
\begin{array}{l}
f^{M}(v;\lambda ,\gamma )=\frac{\gamma ^{2}}{v+\gamma
	/\lambda };\\
%\vspace{-2mm}
v^{\ast }(\lambda ,\gamma ;\ell ) =\left\{
\begin{array}{c}
1,\textrm{ if}\ \ell \leq \left( \frac{\lambda \gamma }{\lambda +\gamma }%
\right) ^{2} \\
0,\textrm{ if}\ \ell \geq \lambda ^{2} \\
\gamma \left( \frac{1}{\sqrt{\ell }}-\frac{1}{\lambda }\right) ,\textrm{
	otherwise}.
\end{array}
\right.
\end{array}
\end{equation*}
Polynomial Soft Weighting:
\begin{equation*}
\begin{array}{l}
f^{P}(v;\lambda,t )=\lambda (\frac{1}{t}||v||_2^{t}-\sum_{i=1}^nv_i)\textrm{; }\\
%\vspace{2mm}
v^{\ast }(\lambda,t ;\ell) =\left\{
\begin{array}{c}
(1-\ell/\lambda)^{1/(t-1)},\textrm{ if}\ \ell<\lambda  \\
0,\textrm{ if}\ \ell\geq \lambda.
\end{array}%
\right.
\end{array}
\end{equation*}

Based on (\ref{eq:selfpaced_obj}), many variations of self-paced learning regime have been proposed, such as self-paced re-ranking \cite{jiang2014easy}, self-paced learning with diversity \cite{jiang2014self}, self-paced curriculum learning \cite{jiang2015self} and self-paced multiple-instance-learning \cite{zhang2016co}. And applications of self-paced learning in many machine learning and computer version tasks, like objective detector adaptation \cite{tang2012shifting}, long-term tricking \cite{supancic2013self}, visual category discovery \cite{lee2011learning}, face identification \cite{lin2017active} and specific-class segmentation learning \cite{kumar2011learning}, have demonstrated its effectiveness especially its robustness when dealing with severally  corrupted data.

Meng and Zhao \cite{meng2015objective} provided some theoretical evidences to illustrate the insights under self-paced learning. They proved that the ASS algorithm to solve the SPL problem exactly accords with the majorization minimization (MM) algorithm implemented on a latent nonconvex SPL objective function. Their work laid the theoretical foundation for SPL.

Comparing SPL with AdaBoost, we can see that there is one thing in common between them, that is, training samples with different losses are unequal and will be endowed with different weights. The way of SPL and AdaBoost to assign weights to samples is very different. For AdaBoost, samples with large losses are paid more attention to and are given larger weights. On the contrary, for SPL, samples with losses larger than a certain constant are thought as outliers and their weights are zero. On account of the fact that the reason why AdaBoost is not robust is that AdaBoost assigns too large weights to the samples with very large losses and those samples are usually outliers, we expect that SPL can provide a complementary assistance to restrict the sample weights of AdaBoost and improve its robustness. A new robust boosting algorithm named SPLBoost is proposed based on this new idea. We will introduce the details of SPLBoost and then present some theoretical results in the next section.

\section{SPLBoost}
\subsection{Algorithm}
Given training data $\left( x_1,y_1\right) $,\ldots,$\left( x_n,y_n\right) $, where $ x_i $ is a vector-valued feature and $ y_i\in \left\lbrace 1,-1\right\rbrace  $. It is known that AdaBoost algorithm builds an additive logistic regression model for minimizing the expected risk based on the exponential loss function $ \varphi \left( yf\left( x\right) \right) =\exp \left( -yf\left( x\right) \right)  $. In each iteration, supposing that  we have a current classifier $ F(x) $, AdaBoost then seeks a new weak learner $ f(x) $ through the following optimization problem:
\begin{equation}\label{eq:AdaBoost optimize iteration }
\{\alpha, f \}=\arg \min_{\alpha, f} \sum_{i=1}^{n} \mathrm{e} ^{-y_i(F(x_i)+\alpha f(x_i))}.
\end{equation}
Embeding (\ref{eq:AdaBoost optimize iteration }) in a general SPL model, we can get a new  algorithm named SPLBoost, which seeks a new weak learner $ f(x) $ and updates its weight $ \alpha $ in the final strong learner in every step:
\begin{equation}\label{eq:SPLBoost optimize iteration }
\{\alpha, f, \mathbf{v} \}=\arg \min_{\alpha, f} \sum_{i=1}^{n} v_i\mathrm{e} ^{-y_i(F(x_i)+\alpha f(x_i))}+\hat {f} (v_i,\lambda), 
\end{equation}
where  $\hat{f} $ is a SP-regularizer and $ \mathbf{v}=[v_1,\cdots,v_n]^T $ is the latent weight variable induced from $\hat{f} $, and $ \lambda $ is the ``age'' parameter. It is easy to see that the objective (\ref{eq:SPLBoost optimize iteration }) of SPLBoost in each step  is to minimize a weighted exponential loss together with a self-paced regularizer. In AdaBoost,  the exponential loss is directly minimized and the outliers whose losses are usually very large are easy to be paid more attention to. SPLBoost overcomes this problem by assigns different weights $ v $  to the exponential losses of training samples. Although different SP-regularizer can induce different format of $ v $ ,  $ v $ will always be zero when the loss is very large, which usually means that the corresponding sample is very likely to be outlier. In this way, SPLBoost can eliminate the negative influence of the outliers in training data to a large extent and improve the robustness of AdaBoost.

As in SPL, we shall solve (\ref{eq:SPLBoost optimize iteration }) by alternative search strategy (ASS), which is a popular iterative process. In order to distinguish the iterative process of the ASS from the iterative process of SPLBoost for updating classifier $F(x)$, we call the former inner iteration, and the latter outer iteration. In every inner iteration, with fixed $ \alpha $ and $ f $, (\ref{eq:SPLBoost optimize iteration }) is a optimization problem about $ \mathbf{v} $ with global optimum  $ \mathbf{v^*}=[v_1^*,\cdots,v_n^*]^T $ whose  form is different for different SP-regularizer and has been presented in relevant papers \cite{kumar2010self, jiang2014easy, li2016multi}. Especially, When the SP-regularizer is the negative $ l_1 $-norm, or say, Hard Weighting, we can plug the above exponential loss into the formula (\ref{solve V }) and then $ \mathbf{v^*} $ can be easily calculated by the following formulation: 

\begin{equation}\label{solve V in SPLBoost }
v_i^*=\begin {cases}
1,&\mathrm{e} ^{-y_i(F(x_i)+\alpha f(x_i))}<\lambda,\\
0,&\text{otherwise.}
\end{cases}
\end{equation}
With fixed $ \mathbf{v} $, (\ref{eq:SPLBoost optimize iteration }) is a weighted exponential loss minimization problem:
 \begin{equation}\label{eq:SPLBoost solve f }
\{\alpha, f \}=\arg \min_{\alpha, f} \sum_{i=1}^{n} v_i\mathrm{e} ^{-y_i(F(x_i)+\alpha f(x_i))}.
\end{equation}
The above (\ref{eq:SPLBoost solve f })  is similar to the minimization problem (\ref{eq:AdaBoost optimize iteration }),  thus we can solve it by using the similar idea of AdaBoost. Follow the way in \cite{friedman2000additive}, for fixed $ \alpha
 $, we expand (\ref{eq:SPLBoost solve f }) to second order about $ f(x_i)=0 $: 
\begin{equation}\label{eq:SPLBoost solve f for fixed alpha }
\begin{split}
f&=\arg \min_{f} \sum_{i=1}^{n} v_i\mathrm{e} ^{-y_i(F(x_i)+\alpha f(x_i))}\\
&\approx \arg \min_{f} \sum_{i=1}^{n} v_i\mathrm{e} ^{-y_i(F(x_i))}(1-y_i\alpha f(x_i)+\alpha ^2f(x_i)^2/2).
\end{split}
\end{equation}
Taking $ f(x)\in \{-1,1\} $ into consideration, we have
\begin{equation}\label{eq2:SPLBoost solve f for fixed alpha }
\begin{split}
f&=\arg \min_{f} \sum_{i=1}^{n} v_i\mathrm{e} ^{-y_i(F(x_i))}(y_i-\alpha f(x_i))^2\\
&= \arg \min_{f} \sum_{i=1}^{n} v_i\mathrm{e} ^{-y_i(F(x_i))}(y_i-f(x_i))^2\\
&= \arg \min_{f} \sum_{i=1}^{n} v_iw_i(y_i-f(x_i))^2, 
\end{split}
\end{equation}
where $w_i=\mathrm{e} ^{-y_i(F(x_i))}$ which is the same as the  sample weights defined in AdaBoost. We can find that solving the above weighted least squares problem to obtain the weak learner $ f(x) $ is equivalent  to implement  AdaBoost except for the use of latent weight variable $ \mathbf{v} $. Thus, imitating the approach in AdaBoost, one can train $ f(x) $ from the training data with sample weights $ v_iw_i $ rather than $ w_i $. Given $ f(x)\in \{-1,1\} $, we can directly optimize (\ref{eq:SPLBoost solve f }) to determine $ \alpha $:
\begin{equation}\label{eq:SPLBoost solve alpha for fixed f }
\begin{split}
\alpha &=\arg \min_{\alpha } \sum_{i=1}^{n} v_iw_i\mathrm{e} ^{-y_i\alpha f(x_i)}\\
&= \arg \min_{\alpha } \sum_{y_i=f(x_i)} v_iw_i\mathrm{e} ^{-\alpha}+\sum_{y_i\neq f(x_i)} v_iw_i\mathrm{e} ^{\alpha}.
\end{split}
\end{equation}
It is not hard to see the above objective is a convex function, thus to get the optimal $ \alpha $, we can directly calculate its derivative and set it to be zero:
\begin{equation}\label{derivative is zero}
-\sum_{y_i=f(x_i)} v_iw_i\mathrm{e} ^{-\alpha}+\sum_{y_i\neq f(x_i)} v_iw_i\mathrm{e} ^{\alpha}=0.
\end{equation}
From (\ref{derivative is zero}), we have
\begin{equation}\label{form of alpha}
\begin{split}
\alpha &=\frac{1}{2}\log \frac{\sum_{y_i=f(x_i)} v_iw_i}{\sum_{y_i \neq f(x_i)} v_iw_i}\\
&=\frac{1}{2}\log \frac{1-\sum_{y_i \neq f(x_i)} v_iw_i}{\sum_{y_i \neq f(x_i)} v_iw_i}\\
&=\frac{1}{2}\log \frac{1-\mathrm{err}}{\mathrm{err}}, 
\end{split}
\end{equation}
where $ \mathrm{err}=\sum_{y_i \neq f(x_i)}v_iw_i $ is the weighted misclassification error of weak learner $ f(x) $. It is easy to see that the formula of $ \alpha $ in (\ref{form of alpha}) is also consistent with AdaBoost except for the latent weight variable $ \mathbf{v} $. 

By alternative iteration, we can calculate the optimal $ \mathbf{v} $, $ \alpha $ and $ f $ for (\ref{eq:SPLBoost optimize iteration }). Then we can get the update for $ F(x) $:
\begin{equation}
F(x) \leftarrow F(x)+\alpha f(x).
\end{equation} 
In the next outer iteration, the latent weight variable $ \mathbf{v} $ can be initialized to the current $ \mathbf{v} $ and $ w_i $ is updated as follows:
\begin{equation}\label{eq1:update w}
w_i \leftarrow w_i\mathrm{e}^{-\alpha y_if(x_i)}. 
\end{equation}
Since $ -y_if(x_i)=2\times 1_{y_i\neq f(x_i)}-1 $, the update is equivalent to
\begin{equation}\label{eq2:update w}
w_i \leftarrow w_i\mathrm{exp}\left( \log \frac{1-\mathrm{err}}{\mathrm{err}}1_{y_i\neq f(x_i)}\right).
\end{equation}
This update of $w$ is clearly the same as that in AdaBoost.

\begin{algorithm}[t]
	\caption{SPLBoost algorithm} %算法的名字	
	\label{SPLBoost algorithm}
	\hspace*{0.02in} {\bf Input:} %算法的输入， \hspace*{0.02in}用来控制位置，同时利用 \\ 进行换行
	training samples $ \{(x_1,y_1),\cdots, (x_n,y_n)\}$, iteration \\
	\hspace*{0.45in} {count} $ T $, parameter $ \lambda $; \\
	\hspace*{0.02in} {\bf Initialization:} $ w_i=1/n $, $ v_i=1 $, $ i=1,\cdots, n $;
	\begin{algorithmic} 
		\For{t=1 to $ T $}
		\While{not converge} % While语句，需要和EndWhile对应
		\State 1. Fit the classifier $ f_t(x) \in \{-1,1\} $ using weights 
		\State \hspace*{0.1in} {$ v_iw_i/(\sum_iv_iw_i)$} on the training data;
		\State 2. Compute $ \mathrm{err}=\sum_{y_i \neq f(x_i)}v_iw_i/\sum_iv_iw_i $,
		\State \hspace*{0.17in}{$ \alpha _t=\frac{1}{2}\log \frac{1-\mathrm{err}}{\mathrm{err}} $}; 
		\State 3. Compute $ \mathbf{v} $;
		\EndWhile
		\State \textbf{end while}
		\State 4. Set $ w_i \leftarrow w_i\mathrm{exp}\left( \log \frac{1-\mathrm{err}}{\mathrm{err}}1_{y_i\neq f(x_i)}\right) $;
		\EndFor
		\State \textbf{end for}
	\end{algorithmic}
		\hspace*{0.02in} {\bf Output:} %算法的结果输出
		The strong classifier $ sign\sum_{t=1}^{T}\alpha _tf_t(x) $;
\end{algorithm}

The details of SPLBoost are summarized in Algorithm \ref{SPLBoost algorithm}. Next, we shall give some remarks about it. 

Firstly, there are two layers of iteration in Algorithm 1: the outer iteration is to update the classifier $ F(x) $, and the inner iteration is to find the optimal latent weight variable $ \mathbf{v} $, the weak learner $ f(x) $ and its weight $ \alpha $ in the current outer iteration. When a new inner iteration starts, the latent weight variable $ \mathbf{v} $ is initialized to the optimal value provide by the last outer iteration and our experiments show that in this case, the inner iteration is rapidly converge and there are no significant difference between the case where the inner iteration runs only one step and the case where the inner iteration keeps running until converged. As such, we actually set the inner iteration as one step to speedup the algorithm implementation. 

Secondly, for choosing the age parameter $ \lambda $, it is easy to see that when the losses of samples are larger than $ \lambda $, the latent weight variable of those samples could be zero, which means that those samples would not be selected during training  process. Thus $ \lambda $ actually represents the ``tolerance'' of the algorithm toward noises and outliers. The larger $ \lambda $ is, the more ``tolerant'' the algorithm is to noises and outliers and the less the samples which are considered to be outliers that would be abandoned ($ v_i=0 $). Furthermore,  when $ \lambda $ is large enough, the SPBoost algorithm degenerates into AdaBoost. On the contrary, the smaller $ \lambda $ is, the more ``rigorous'' the algorithm is to large noises and outliers and more samples would be abandoned. Apparently, the value of $ \lambda $ has a huge influence on the performance of the algorithm and thus it is important to select a appropriate $ \lambda $. In practice, we usually select the proper $ \lambda $ via cross validation.

Finally, since that the weak learner $ f(x) $ produced by Algorithm \ref{SPLBoost algorithm} is restricted in $ \{-1,1\} $,  the losses of the samples can only be $ \mathrm{e}^{-\alpha} $ or $ \mathrm{e}^\alpha $ when we calculate $ \mathbf{v}$ in the first outer iteration step, where $ \alpha $ is the weight of the first weak learner. Considering that $ v_i=0$ for the samples whose losses are larger than $ \lambda $, the samples that are misclassified by the first weak learner could all not be selected for training in the next outer iteration if $ \lambda $ falls between $[\mathrm{e}^{-\alpha} , \mathrm{e}^\alpha]$, which usually means that too many samples would be abandoned due to the low accuracy of the weak learner. To avoid this kind of unreasonable situation, one can adopt a warm start produre, i.e., in the first few outer iteration steps let $ \lambda $ be a very large number instead of the input value, and after obtaining the corresponding weak learners, $ \lambda $ would be reset to be the input value. As such,  the first few weak learners are trained by AdaBoost and as a result, the samples are determined whether they should be selected or not based on a classifier that is not so bad. According to our experience, it is reasonable  that in the first three outer iterations $ \lambda $ is set to be $10^6$, and then the satisfactory $ \lambda $ can be tuned in $ [1.0,6.0] $ by cross validation.

\begin{figure}
	%\vspace{-4mm}
			\centering
			\subfigure[sample weight of various boosting algorithms]{
				\label{fig:sample weight a}
			\includegraphics[width=\linewidth]{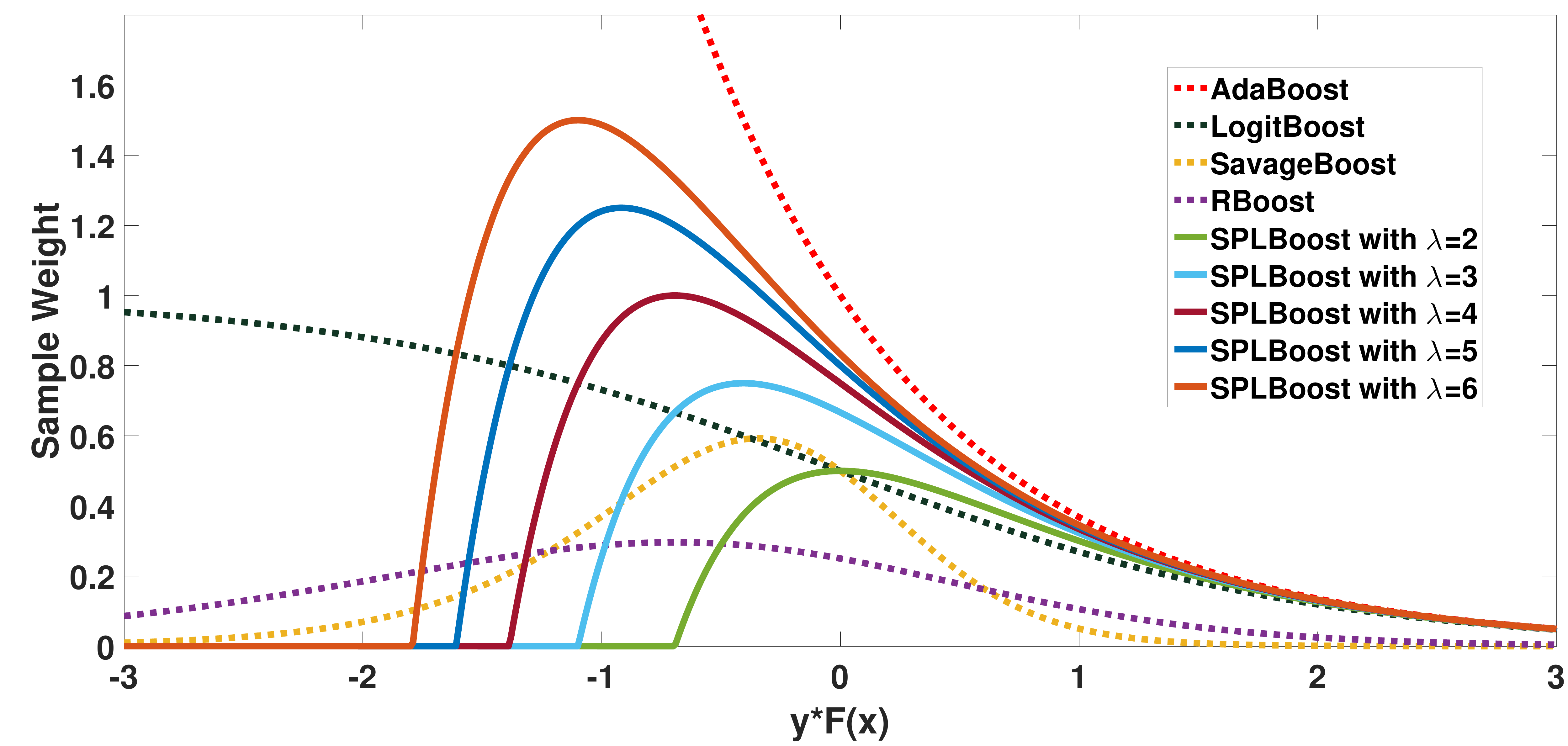}}\\
			%		\scalebox{0.6}[0.55]{\includegraphics{SPLDemo.pdf}}
			%\end{center}
	
		\centering
		\subfigure[sample weight of SPLBoost under various SP-regularizers]{
			\label{fig:sample weight b}
		\includegraphics[width=\linewidth]{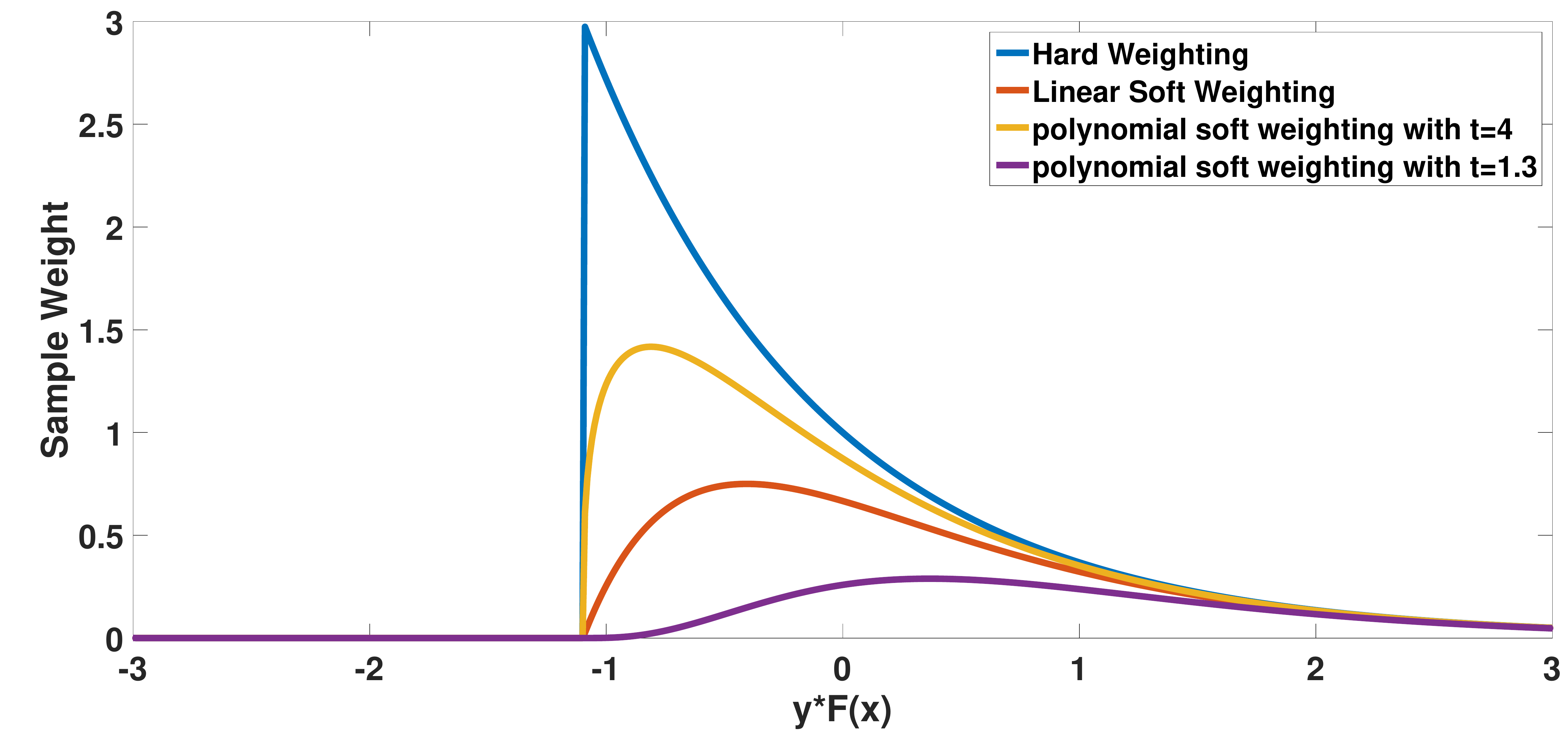}}
	\caption{(a) illustrates the sample weight of various boosting algorithms, including AdaBoost, LogitBoost, SavageBoost, RBoost and SPLBoost with $ \lambda=2,3,4,5,6 $, where the SP-regularizer is Linear Soft Weighting. (b) illustrates the sample weight of SPLBoost with various SP-regularizers, including Hard Weighting, Linear Soft Weighting, Polynomial Soft Weighting with $ t=4 $ and $ t=1.3 $, where $ \lambda $ is fixed as 3.
	}\label{fig:sample weight}
\end{figure}  

Different from AdaBoost where the sample weight is $ w_i=\mathrm{e} ^{-y_i(F(x_i))} $, SPLBoost modifies the sample weight to be $ v_iw_i $ by introducing the latent weight variable $ \mathbf{v} $. We shall show that this is the reason why SPLBoost is more robust. Fig.\ref{fig:sample weight a} illustrates the sample weight of different boosting algorithms, including AdaBoost, LogitBoost, SavageBoost, RBoost and SPLBoost with $ \lambda=2,3,4,5,6 $, where the SP-regularizer is Linear Soft Weighting. It is easy to observe that AdaBoost pays too much attention on the misclassified samples with very large margins which are usually outliers. Thus, AdaBoost is usually very sensitive to outliers. In LogitBoost the weights of the misclassified samples with very large margins are smaller than that in AdaBoost, but they are still larger than the weights of any other samples, thus LogitBoost is still easily affected by outliers. For two popular robust boosting algorithms with non-convex loss functions, i.e., SavageBoost and RBoost,  they give small weights to the misclassified samples with very large margins, thus they are usually insensitive to outliers.  For SPLBoost, if the margin of the misclassified samples is larger than a certain constant which is determined by $ \lambda $, the weights of those samples could be zero. In that way, SPLBoost can more thoroughly eliminate the influence of the always misclassified samples, which are always the outliers.  Fig.\ref{fig:sample weight b} illustrates the sample weight of SPLBoost with different SP-regularizers,  here $ \lambda $ is fixed as 3. One can see that different SP-regularizers provide different distributions of the sample weight which may be more suited to different training data. Some distributions are spiculate and the others are gentle, but they could all be zero when the margins of the misclassified samples are large enough, which guarantees their robustness to outliers.

\subsection{Theoretical Analysis}
In this subsection we shall provide some theoretical analysis of SPLBoost which could show a clear theoretical evidence to clarify why SPLBoost is capable of performing robust especially in outlier/heavy noise cases.

For convenience, we briefly write the exponential loss function $ \mathrm{e}^{-y_i(F(x_i)+\alpha f(x_i))} $ as $ \ell_{i}(\alpha,f)/\ell_i $ and $ \mathrm{e}^{-y(F(x)+\alpha f(x))} $ as $ \ell(\alpha,f)/\ell $ in the following. According to the Theorem 1 in \cite{meng2015objective}, it can be derived that, for latent weight variable $ v^* (\lambda;l) $ conducted by an SP-regularizer and $ \tilde{F}_\lambda(\ell) $ calculated by $ \tilde{F}_\lambda(\ell)=\int_{0}^{\ell} v^* (\lambda;l) \, \mathrm{d}l $, 
and given a fixed $ \alpha ^* $ and $ f^* $, it holds that:
\begin{equation}
\tilde{F}_\lambda(\ell(\alpha,f))\leq Q_\lambda(\alpha,f\mid \alpha ^*,f^*),
\end{equation}
where 
\begin{equation}
\begin{split}
Q_\lambda(\alpha,f\mid \alpha ^*,f^*)= &\tilde{F}
_\lambda(\ell(\alpha ^*,f^*)) +\\
& v^* (\lambda;\ell(\alpha ^*,f^*))(\ell(\alpha ,f)-\ell(\alpha ^*,f^*)).
\end{split}	
\end{equation}
With this, denote
\begin{equation}
\begin{split}
Q_\lambda ^{(i)}(\alpha,f\mid \alpha ^*,f^*)= &\tilde{F}
_\lambda(\ell _i(\alpha ^*,f^*))+\\
& v^* (\lambda;\ell _i(\alpha ^*,f^*))(\ell _i(\alpha ,f)-\ell _i(\alpha ^*,f^*)),
\end{split}
\end{equation}
and we can then easily get that
\begin{equation}\label{eq:surrogate function}
\sum_{i=1} ^n  \tilde{F}_\lambda(\ell_i(\alpha ^*,f^*))\leq \sum_{i=1} ^n  Q_\lambda ^{(i)}(\alpha,f\mid \alpha ^*,f^*), 
\end{equation}
which verifies that $ \sum_{i=1} ^n  Q_\lambda ^{(i)}(\alpha,f\mid \alpha ^*,f^*) $ can be used as a surrogate function of $ \sum_{i=1} ^n  \tilde{F}_\lambda(\ell_i(\alpha ^*,f^*)) $ in the MM algorithm. We can then ready to prove the following result.

\begin{theorem}\label{the:objective function}
	The SPLBoost algorithm is equivalent to the MM algorithm on a minimization problem of the latent SPLBoost objective $ \sum_{i=1}^{n}\tilde{F}_\lambda (\mathrm{e}^{-y_iF(x_i)}) $ with the latent loss  $ \tilde{F}_\lambda (\mathrm{e}^{-yF(x)}) $. 
\end{theorem}
\begin{proof}
	Assume that we have completed $(t-1)$ times outer iteration and get the classifier $ F_{t-1} $. Denote $ f_t^k $ and $ \alpha _t ^k $ as the weak classifier and its weight learned from the $k$th inner iteration in the $ t $th outer iteration, then such two alternative search steps in the next iteration can be explained as a standard MM scheme. Precisely, there are two cases should be dealt with.\\
	{\bf Case 1.} If the inner iteration  have not converged after the $ k $th step, we denote the surrogate function
	\begin{equation*}
	\begin{split}
    Q_\lambda ^{(i)}(\alpha,f\mid \alpha _t ^k,f_t ^k)
	=&\tilde{F}_\lambda(\ell_i^{t-1}(\alpha _t ^k,f_t ^k)) +v^* (\lambda;\ell_i^{t-1}(\alpha _t ^k,f_t ^k))\\
	&\times(\ell_i^{t-1}(\alpha ,f)-\ell_i^{t-1}(\alpha _t ^k,f_t ^k)),
	\end{split}
	\end{equation*}
	where $ \ell_i^{t-1}(\alpha ,f)=e^{-y_i(F_{t-1}(x_i)+ \alpha f(x_i))}$. \\
	\textit{Majorization step:} To obtain each  $ Q_\lambda ^{(i)}(\alpha,f\mid \alpha _t^k,f_t^k) $, we only need to calculate $ v^* (\lambda;\ell_i^{t-1}(\alpha _t^k,f_t^k)) $
	by solving the following problem under the corresponding SP-regularizer $\hat{f}(v_i,\lambda)$:
	\begin{equation*}
	v^* (\lambda;\ell_i^{t-1}(\alpha _t^k,f_t^k))=\arg \min_{v_i \in [0,1]} v_i\ell_i^{t-1}(\alpha _t^k,f_t^k) + \hat{f}(v_i,\lambda).
	\end{equation*}
	This exactly complies with updating $ \mathbf{v} $ in Algorithm 1.\\
	\textit{Minimization step:} we need to calculate:
	\begin{flalign*}
	&[\alpha _t^{k+1},f_t^{k+1}]&\\
	&=  \arg \min_{\alpha,f} \sum_{i=1}^n  \tilde{F}_\lambda(\ell_i^{t-1}(\alpha _t^k,f_t^k))+v^* (\lambda;\ell_i^{t-1}(\alpha _t^k,f_t^k))&\\
	&\quad\times(\ell_i^{t-1}(\alpha ,f)-\ell_i^{t-1}(\alpha _t^k,f_t^k))&\\
	&=  \arg \min_{\alpha,f} \sum_{i=1}^n v^* (\lambda;\ell_i^{t-1}(\alpha _t^k,f_t^k))\ell_i^{t-1}(\alpha ,f)&
	\end{flalign*}
	which is exactly equivalent to update $ \alpha $ and $ f $ in Algorithm 1.\\
	{\bf Case 2.} If the inner iteration  have converged after the $ k $th step with the finally learned weak classifier $ f_t $ and its weight $ \alpha _t $ , we denote $ F_t=F_{t-1}+ \alpha _tf_t $ and the surrogate function
	\begin{flalign*}
	&Q_\lambda ^{(i)}(\alpha,f\mid \alpha _t,f_t)&\\
	&=\tilde{F}_\lambda(\ell_i^{t-1}(\alpha _t,f_t)) +v^* (\lambda;\ell_i^{t-1}(\alpha _t,f_t))(\ell_i^{t-1}(\alpha ,f)-\ell_i^{t-1}(\alpha _t,f_t))&\\
	&=\tilde{F}_\lambda(\ell_i^{t}(\alpha =0,f)) +v^* (\lambda;\ell_i^{t}(\alpha =0,f))(\ell_i^{t}(\alpha ,f)-\ell_i^{t}(\alpha =0,f))&
	\end{flalign*} 
	where $ \ell_i^t(\alpha ,f)=e^{-y_i(F_{t}(x_i)+ \alpha f(x_i))}, $\\
	\textit{Majorization step:} To obtain each  $ Q_\lambda ^{(i)}(\alpha,f\mid \alpha _t,f_t) $ ,we only need to calculate $ v^* (\lambda;\ell_i^{t}(\alpha =0,f))= v^* (\lambda;\ell_i^{t-1}(\alpha _t,f_t)) $
	by solving the following problem under the corresponding SP-regularizer $ \hat{f}(v_i,\lambda) $ :
	\begin{equation*}
	v^* (\lambda;\ell_i^{t-1}(\alpha _t,f_t))=\arg \min_{v_i \in [0,1]} v_i\ell_i^{t-1}(\alpha _t,f_t) + \hat{f}(v_i,\lambda).
	\end{equation*}
	This exactly complies with updating $ \mathbf{v} $ in Algorithm 1.\\
	\textit{Minimization step:} we need to calculate:
	\begin{flalign*}
	&[\alpha _{t+1}^{1},f_{t+1}^{1}]&\\
	&=\arg \min_{\alpha,f} \sum_{i=1}^n  \tilde{F}_\lambda(\ell_i^{t}(\alpha =0,f)) +v^* (\lambda;\ell_i^{t}(\alpha =0,f))&\\
	&\quad \times(\ell_i^{t}(\alpha ,f)-\ell_i^{t}(\alpha =0,f))&\\
	&=\arg \min_{\alpha,f} \sum_{i=1}^n v^* (\lambda;\ell_i^{t}(\alpha =0,f))\ell_i^{t}(\alpha ,f)&\\
	&=\arg \min_{\alpha,f} \sum_{i=1}^n v^* (\lambda;\ell_i^{t-1}(\alpha _t,f_t))\ell_i^{t}(\alpha ,f)&
	\end{flalign*}
	which is exactly equivalent to the steps of update $ \alpha $ and $ f $ in Algorithm 1.
	
\end{proof}

With Theorem 1, various off-the-shelf theoretical results of MM algorithm can then be used to explain the properties of SPLBoost. Particularly, based on the well-known convergence theory of MM algorithm, that is, the lower-bounded latent SPLBoost objective is monotonically decreasing during the iteration. Thus, a weak convergence result of SPLBoost can be directly obtained.

A number of  formulas of the latent loss $ \tilde{F}_\lambda(\ell)=\int_{0}^{\ell} v^* (\lambda;l) \, \mathrm{d}l $ under various SP-regularizers have been calculated and presented in \cite{li2016multi, meng2015objective}, and we only need to plug the exponential loss function $ \ell =\exp \left( -yF\left( x\right) \right)  $ into those formulas to get the latent SPLBoost losses under various SP-regularizers. Fig. \ref{fig:loss function a} illustrates some popular loss functions, including exponential loss, logistic loss, Savage loss, Savage2 loss, 0-1 loss and latent SPLBoost loss with $ \lambda=2,3,4,5,6 $, where the SP-regularizer is Linear Soft Weighting. It is easy to see from Fig. \ref{fig:loss function a} that,  compared with the original exponential loss function, the latent SPLBoost loss has an evident suppressing effect on the large losses. When the loss is larger than a certain threshold which is determined by the ``age'' parameter $ \lambda $, the latent SPLBoost loss $ \tilde{F}_\lambda (\mathrm{e}^{-yF(x)}) $ becomes a constant thereafter, which rationally explains why SPLBoost shows good robustness to the outliers and heavy noises. The misclassified samples with very large margins will have constant SPLBoost losses and thus have no effect on the model training due to their zero gradients. Corresponding to the original SPLBoost model, the latent weight variable $ v_i $ of those large-loss samples will be 0, and thus those samples will have no influence on the training of the weak learners. $ \lambda $ actually determines the ``degree'' of the suppressing effect SPLBoost loss has on the large losses. The larger $ \lambda $ is, the gentler the suppressing effect is, and vice versa. When $ \lambda =\infty $, the suppressing effect completely disappears and the latent SPLBoost loss under Hard Weighting SP-regularizer degenerates into exponential loss. Fig. \ref{fig:loss function b} illustrates the latent SPLBoost loss under various SP-regularizers, including Hard Weighting, Linear Soft Weighting, Polynomial Soft Weighting with $ t=4 $ and $ t=1.3 $, where $ \lambda $ is fixed as 3. We can see that different SP-regularizers give different shapes of the latent SPLBoost loss, but they will all becomes constant when the loss is larger than a certain constant, which guarantees their robustness to the outliers and heavy noises.  
\begin{figure}
	%\vspace{-4mm}
	\centering
	\subfigure[various loss functions]{
		\label{fig:loss function a}
		\includegraphics[width=\linewidth]{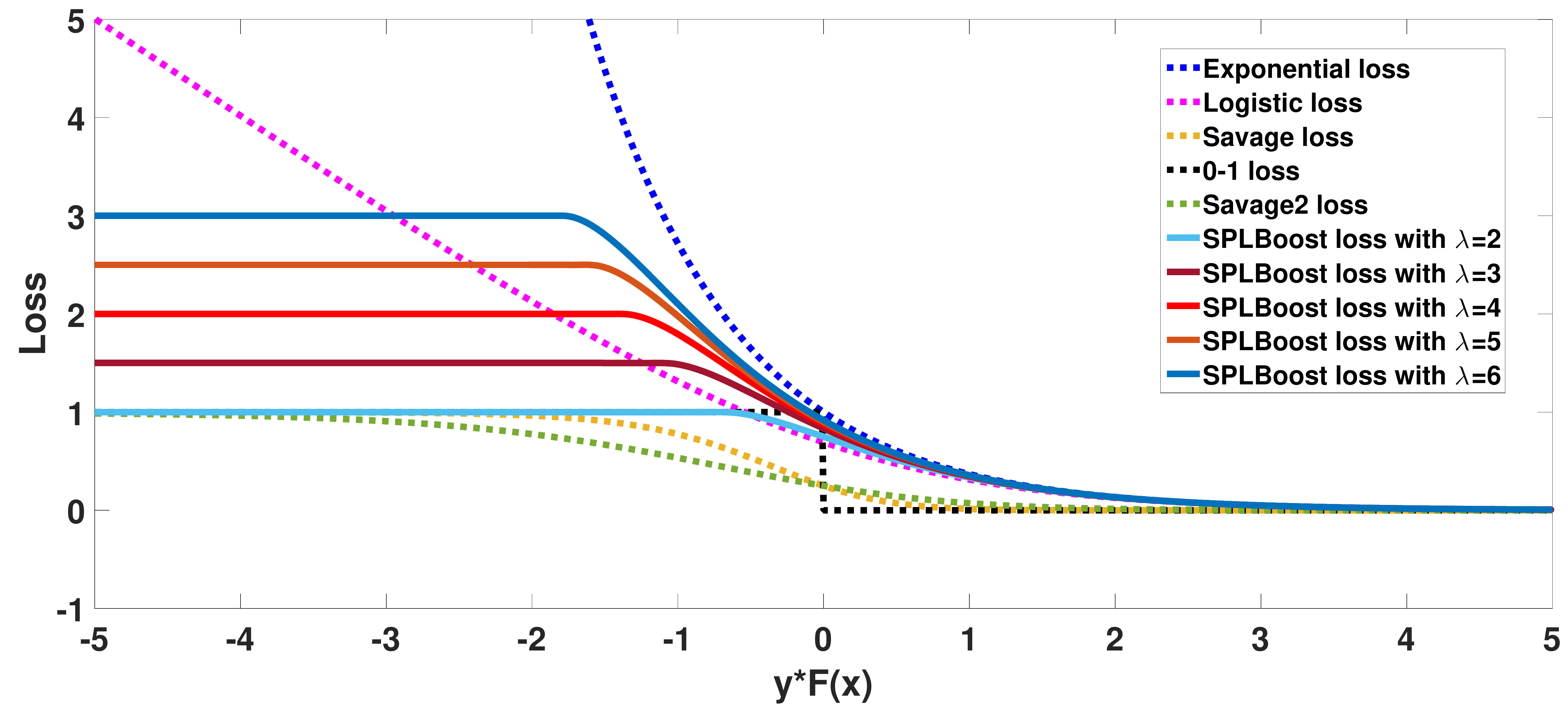}}\\
	%		\scalebox{0.6}[0.55]{\includegraphics{SPLDemo.pdf}}
	%\end{center}
	
	\centering
	\subfigure[latent SPLBoost loss under various SP-regularizers]{
		\label{fig:loss function b}
		\includegraphics[width=\linewidth]{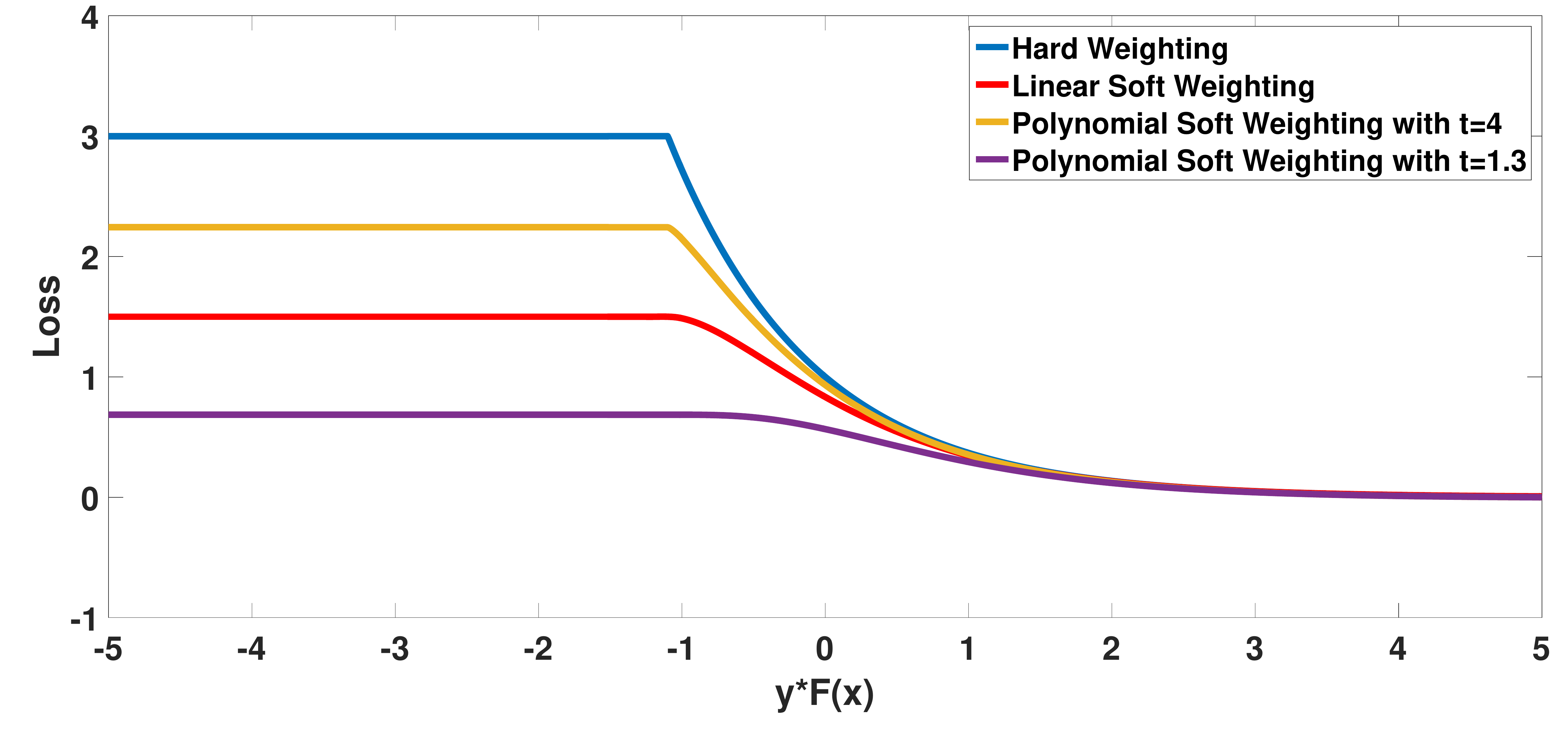}}
	\caption{(a) illustrates various loss functions, including exponential loss, logistic loss, Savage loss, Savage2 loss, 0-1 loss and latent SPLBoost loss with $ \lambda=2,3,4,5,6 $, where the SP-regularizer is Linear Soft Weighting. (b) illustrates the latent SPLBoost loss under various SP-regularizers, including Hard Weighting, Linear Soft Weighting, Polynomial Soft Weighting with $ t=4 $ and $ t=1.3 $, where $ \lambda $ is fixed as 3.
	}\label{fig:loss function}
\end{figure}

According to the above analysis, it is easy to see that SPLBoost is actually a optimization problem with a non-convex loss function. Different from many other robust boosting algorithms which directly optimize the non-convex objective functions, SPLBoost decomposes the minimization of the robust but difficult-to-solve non-convex problem into two much easier optimization problems with respect to the latent weight variable $ \mathbf{v} $ and the weak learner $ f(x) $. In this sense, SPLBoost avoids the difficulty of  non-convex optimization and simplifies the solving of such a problem.   

\section{Experiments}
In this section, we will test the robustness of the proposed SPLBoost algorithm through the thorough experiments in both synthetic and UCI data set.
\subsection{Synthetic Gaussian Data Set}
It has been shown that the loss functions used by the boosting algorithms have an important effect on the robustness of the algorithms to the outliers and heavy noises. Moreover, the reweighting strategy of a boosting
algorithm directly comes from the loss function it used and directly determines how much attention the algorithm pays to the various samples. Thus, a reasonable reweighting strategy is necessary to a robust boosting algorithm and to weaken the interference of the outliers to the training, a good reweighting strategy should give the least possible weights to the outliers. In the synthetic data set, the underlying distribution of the samples and the information of the outliers added to those samples is known, thus it is easy to determine the optimal Bayes decision boundary
for the classification problem and observe the rationality of the distribution of the sample weights. To compare the different reweighting strategies of different boosting algorithms, we first evaluate the proposed SPLBoost algorithm, AdaBoost and some other robust boosting algorithm, including LogitBoost, SavageBoost, RBoost and RobustBoost, on a synthetic Gaussian data set, and to directly visualize the experimental results, the data set is two dimensions. The experimental settings and results are as follows.

According to the 2-D Gaussian distributions
\begin{equation*}
   N\left( [2,-2],
   \begin{bmatrix}
   2.5 & 1.5\\
   1.5 & 5
   \end{bmatrix} 
   \right)  
\end{equation*} 
and
\begin{equation*}
N\left( [-2,2],
\begin{bmatrix}
2.3 & -0.7\\
-0.7 & 2.3
\end{bmatrix} 
\right),  
\end{equation*} 
we first generate 100 samples for both the negative and positive classes, then randomly select 15\% from both the two classes and reverse their labels, and those samples selected can be considered to be outliers. In this way we have obtained the two-class training data with 15\% outliers and then we can train the classifiers using the aforementioned boosting algorithms. For AdaBoost, SPLBoost and RobustBoost in which the weak learner algorithms can be classification methods, the classification tree C$ 4.5 $ is selected to be the weak learner and in SPLBoost the Hard Weighting is used to be the SP-regularizer. And for LogitBoost and RBoost in which the weak learner algorithms are restricted to regression methods, the regression tree CART is used to be the weak learner.   

\begin{figure}[htbp]
	%\vspace{-4mm}
	%\centering
	\subfigure[original data]{
		\label{experiment:original}
		\includegraphics[width=0.45\linewidth]{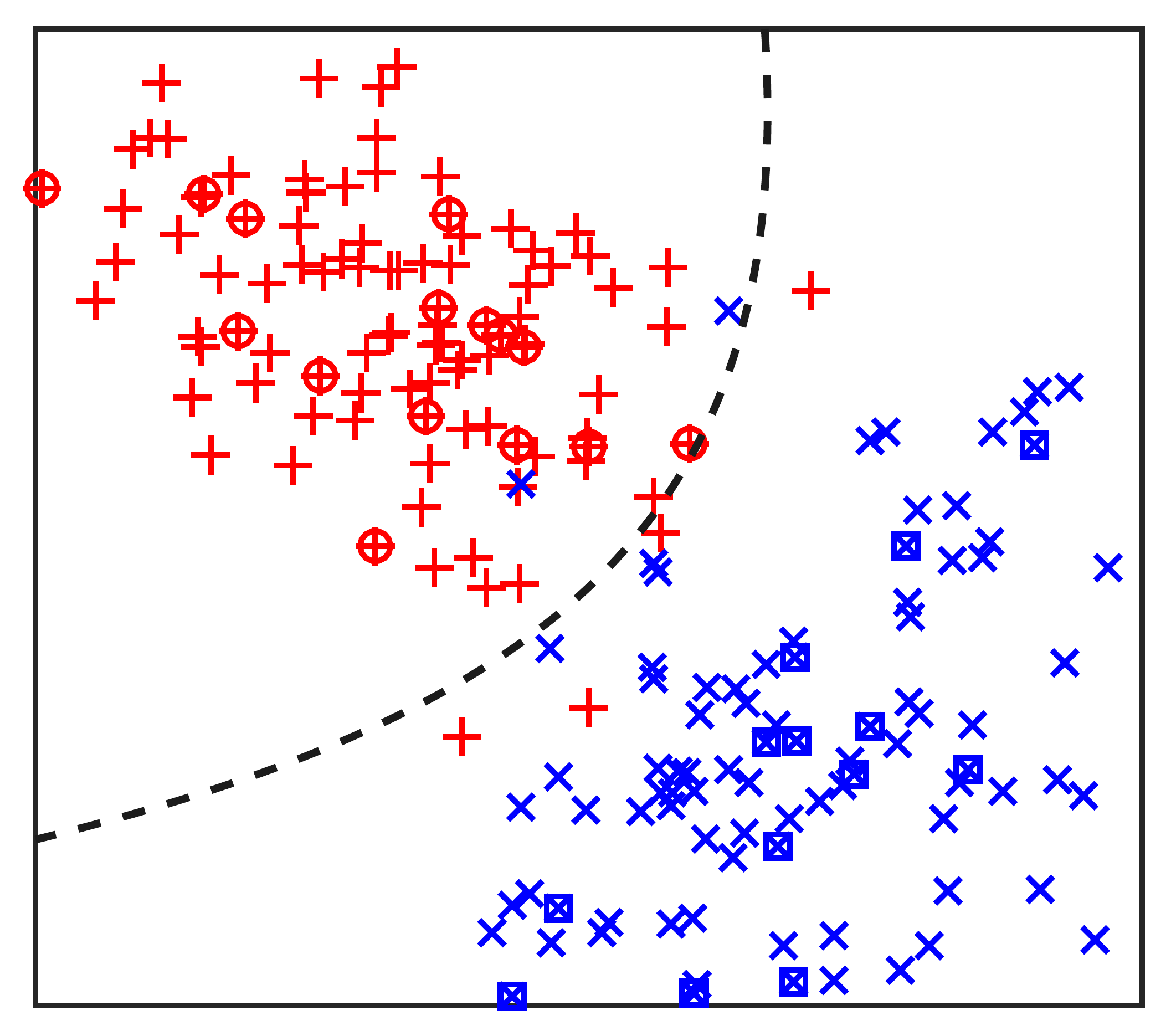}}
	%\hspace{1in}
	\subfigure[AdaBoost]{
		\label{experiment:AdaBoost}
		\includegraphics[width=0.45\linewidth]{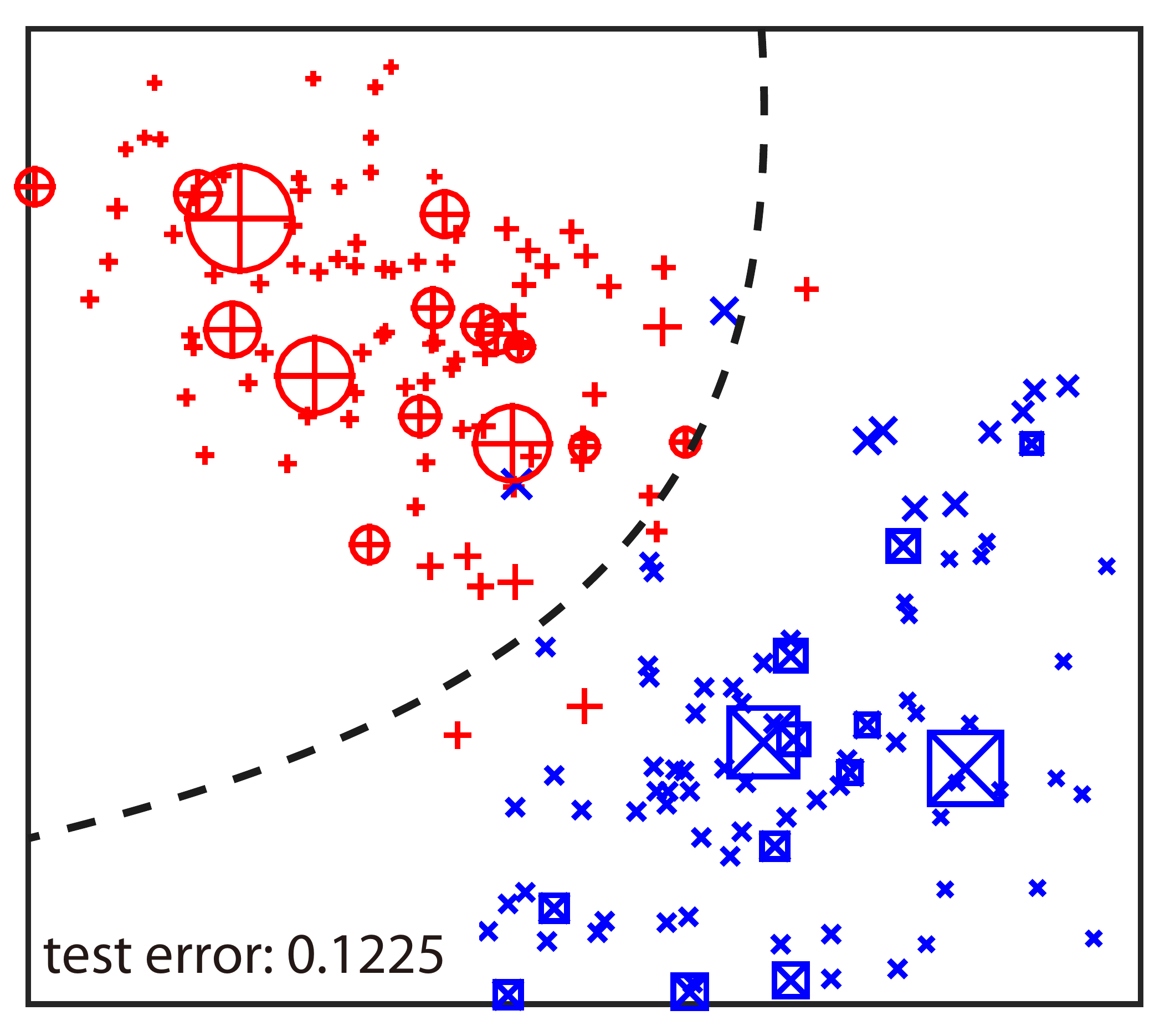}}\\
	\subfigure[LogitBoost]{
		\label{experiment:LogitBoost}
		\includegraphics[width=0.45\linewidth]{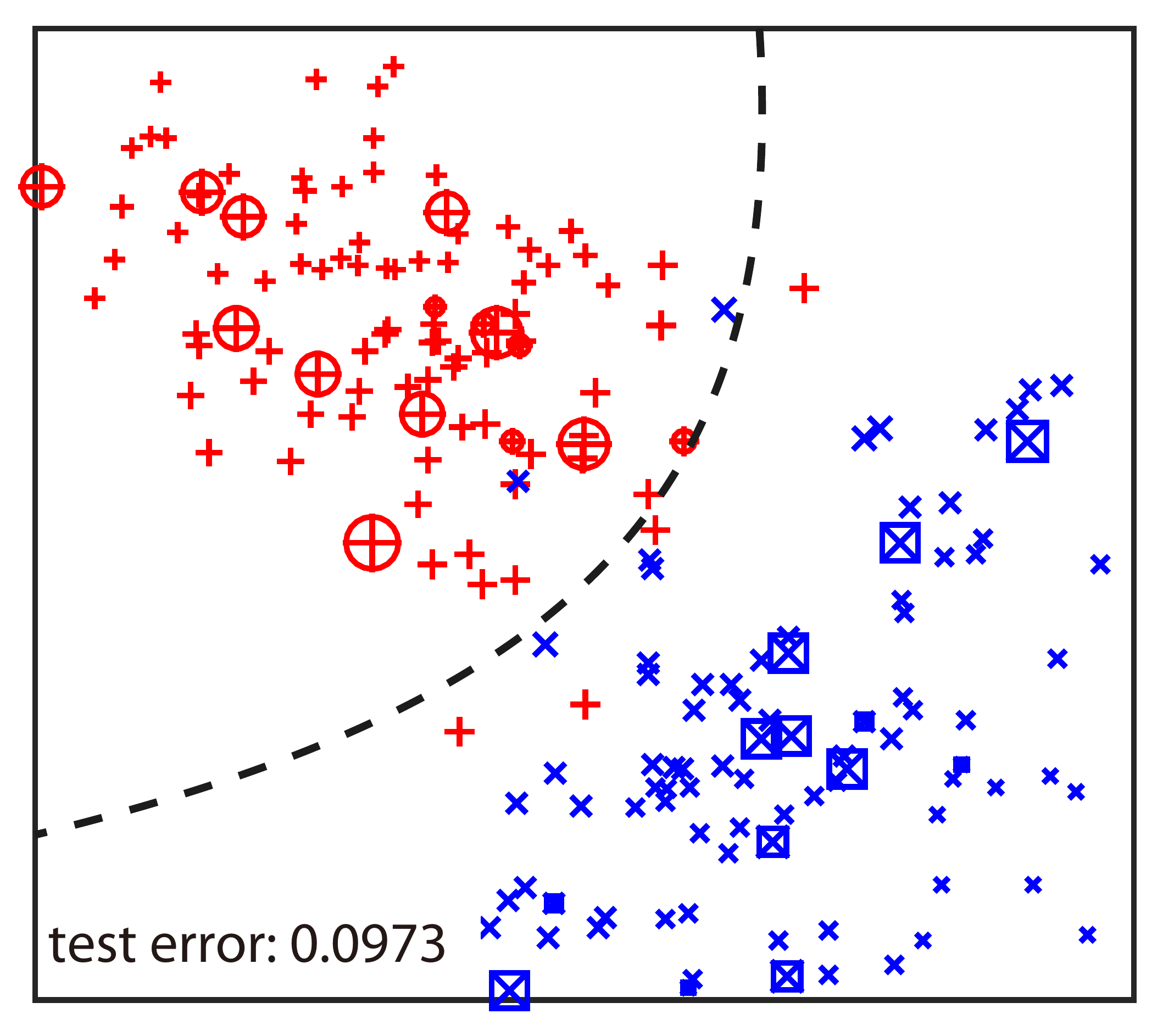}}
	%\hspace{1in}
	\subfigure[SavageBoost]{
		\label{experiment:SavageBoost}
		\includegraphics[width=0.45\linewidth]{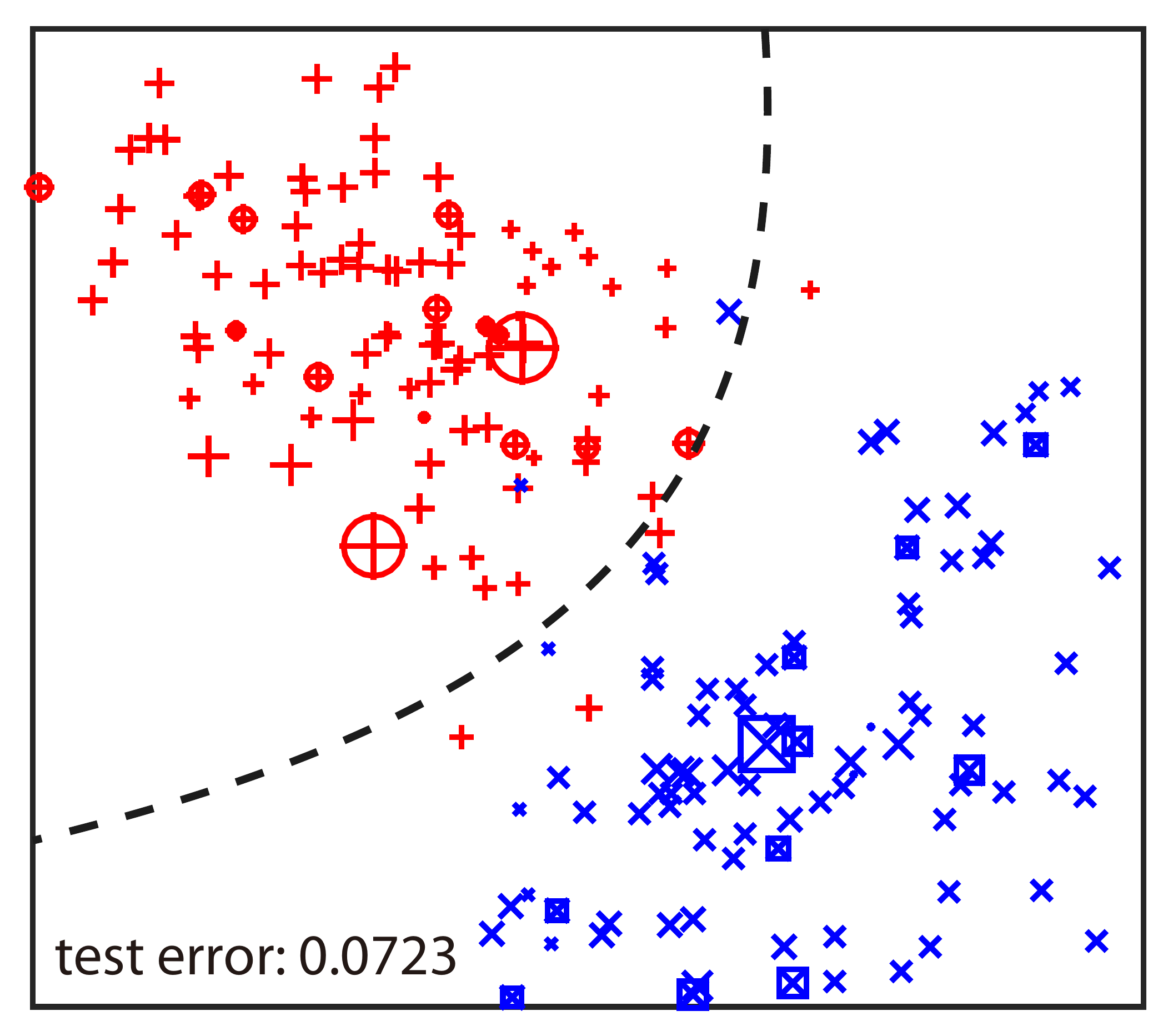}}\\
	\subfigure[RBoost]{
		\label{experiment:RBoost}
		\includegraphics[width=0.45\linewidth]{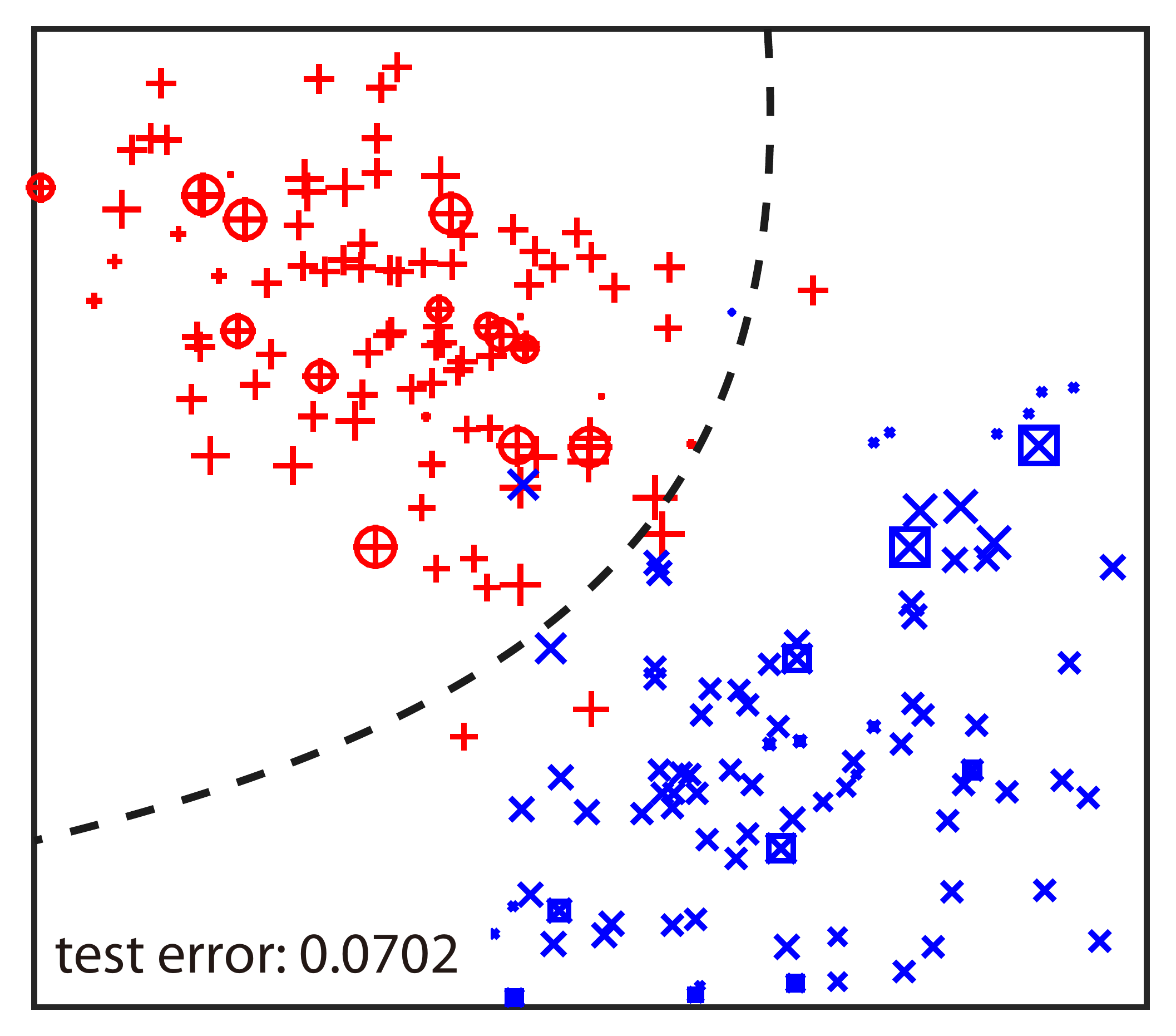}}
	%\hspace{1in}
	\subfigure[RobustBoost]{
		\label{experiment:RobustBoost}
		\includegraphics[width=0.45\linewidth]{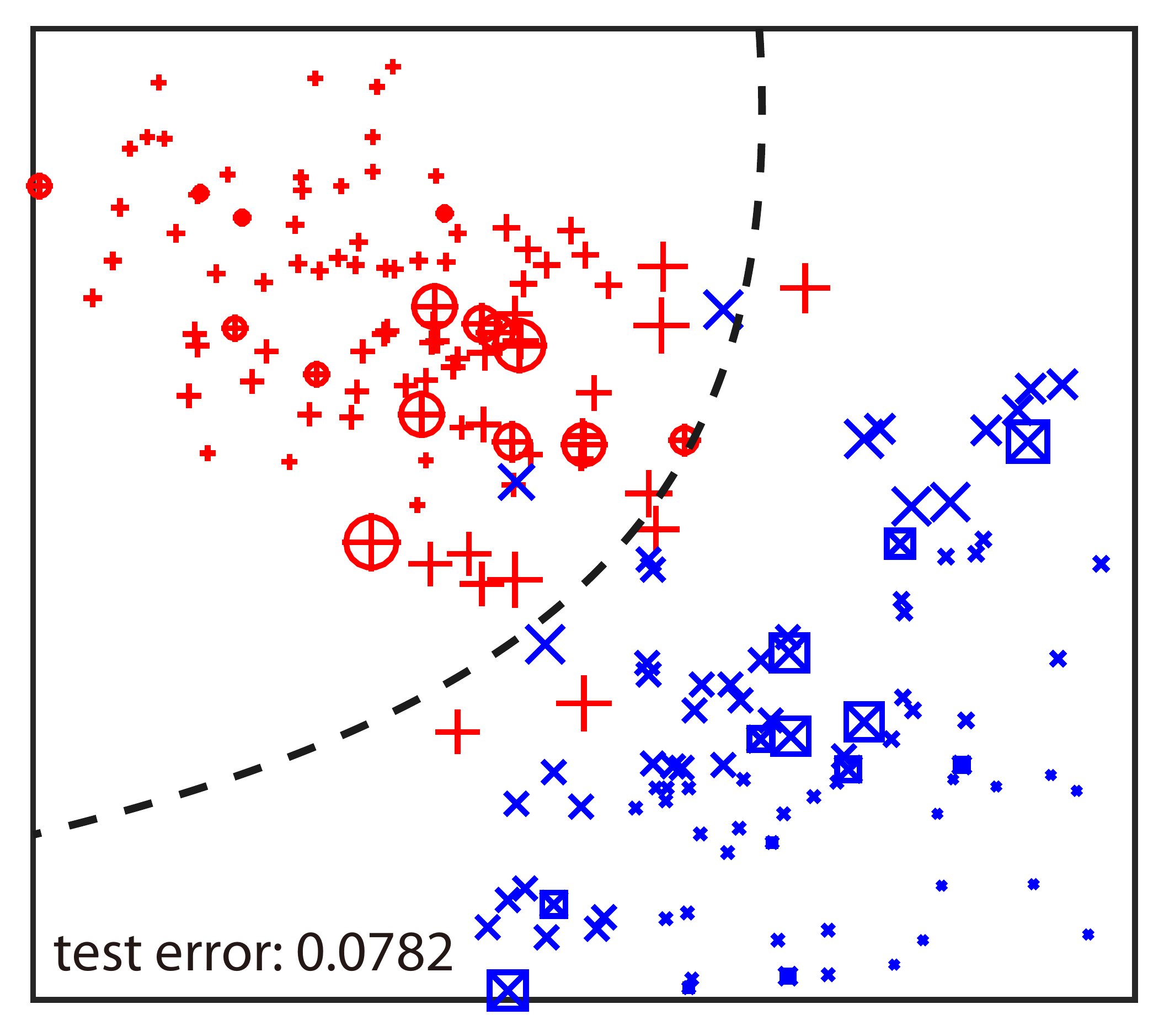}}\\
	\subfigure[SPLBoost with $ \lambda=1.5 $]{
		\label{experiment:SPLBoost with lambda=1.5 }
		\includegraphics[width=0.45\linewidth]{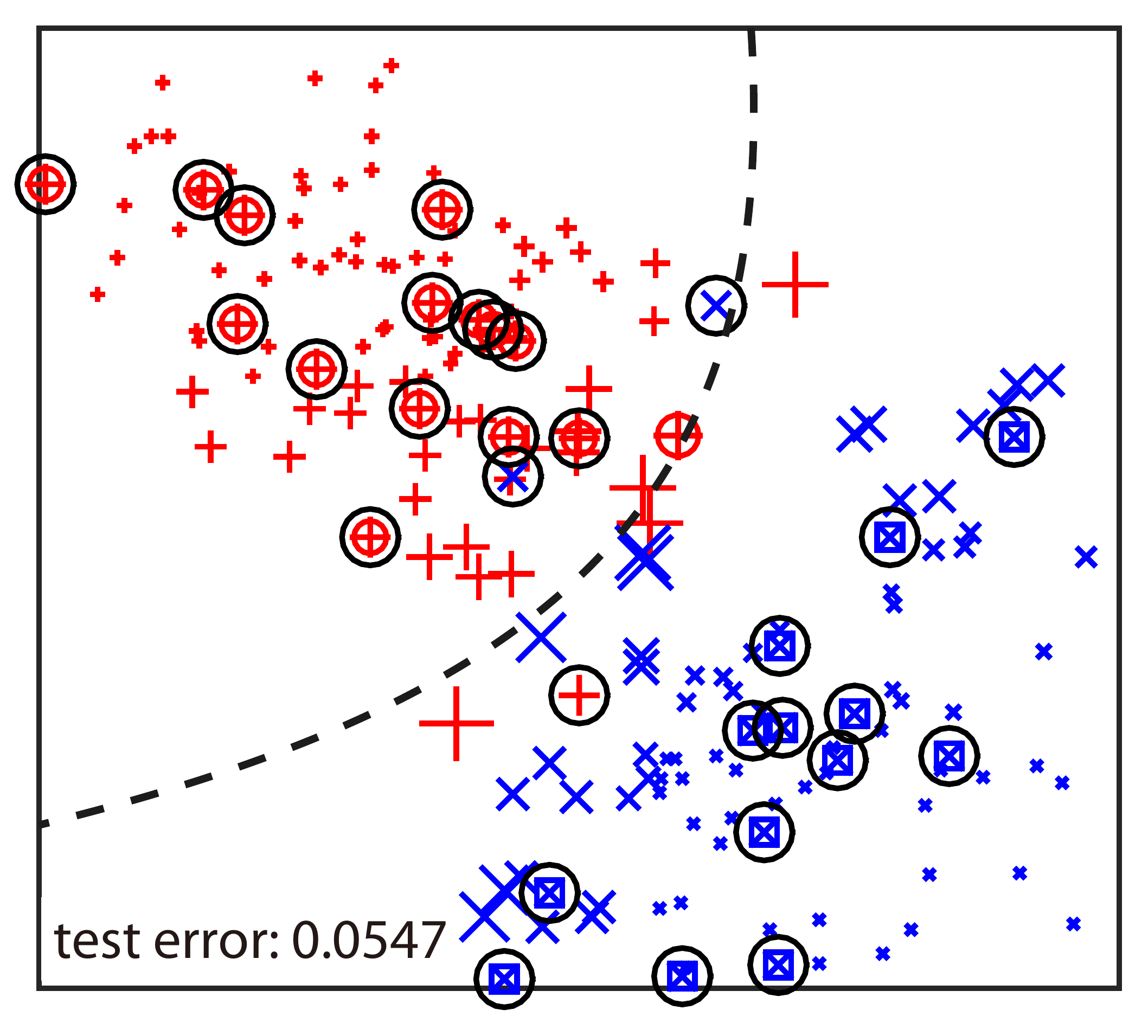}}
	%\hspace{1in}
	\subfigure[SPLBoost with $ \lambda=2.0 $]{
		\label{experiment:SPLBoost with lambda=2}
		\includegraphics[width=0.45\linewidth]{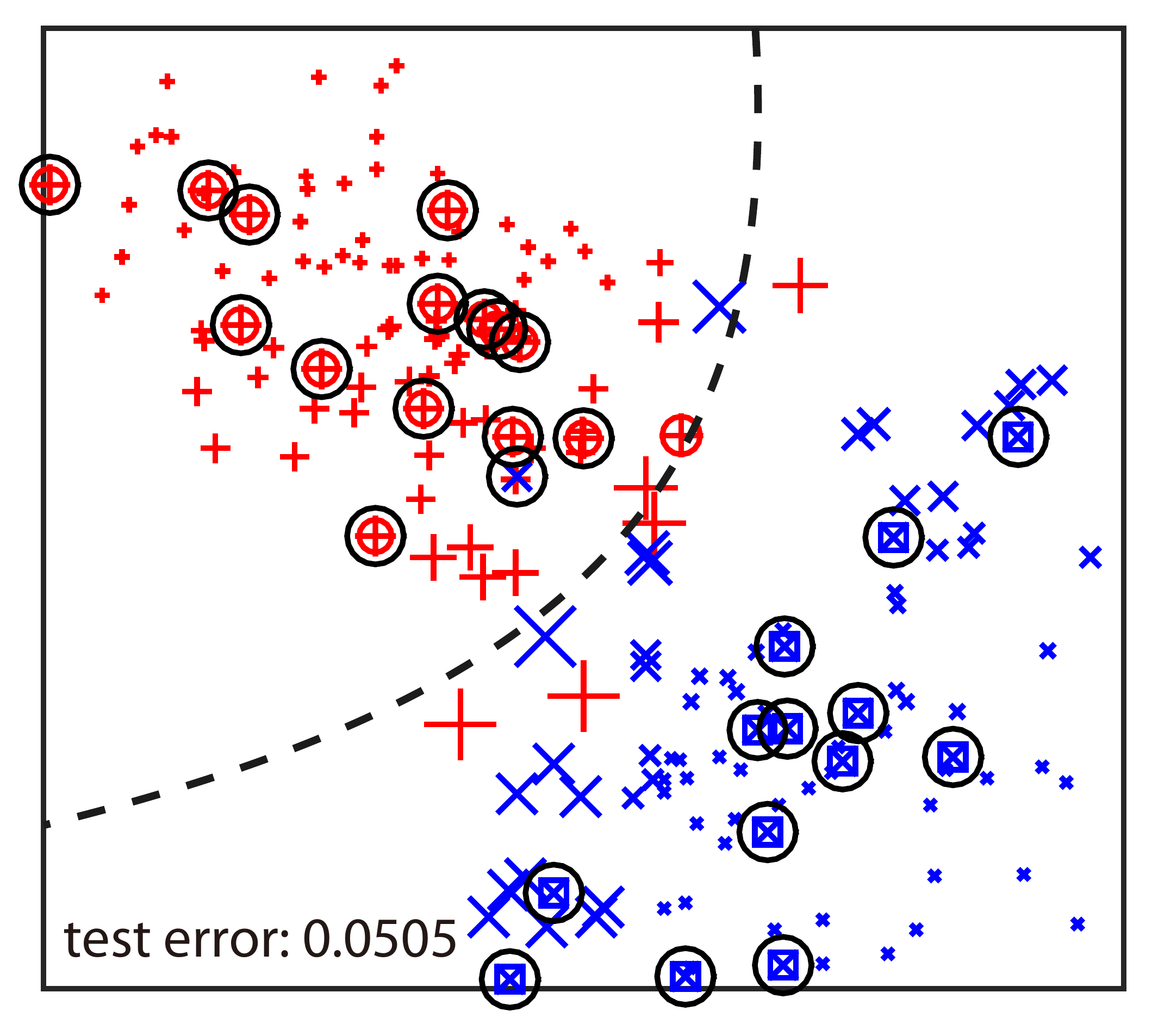}}\\
	\subfigure[SPLBoost with $ \lambda=2.5 $]{
		\label{experiment:SPLBoost with lambda=2.5 }
		\includegraphics[width=0.45\linewidth]{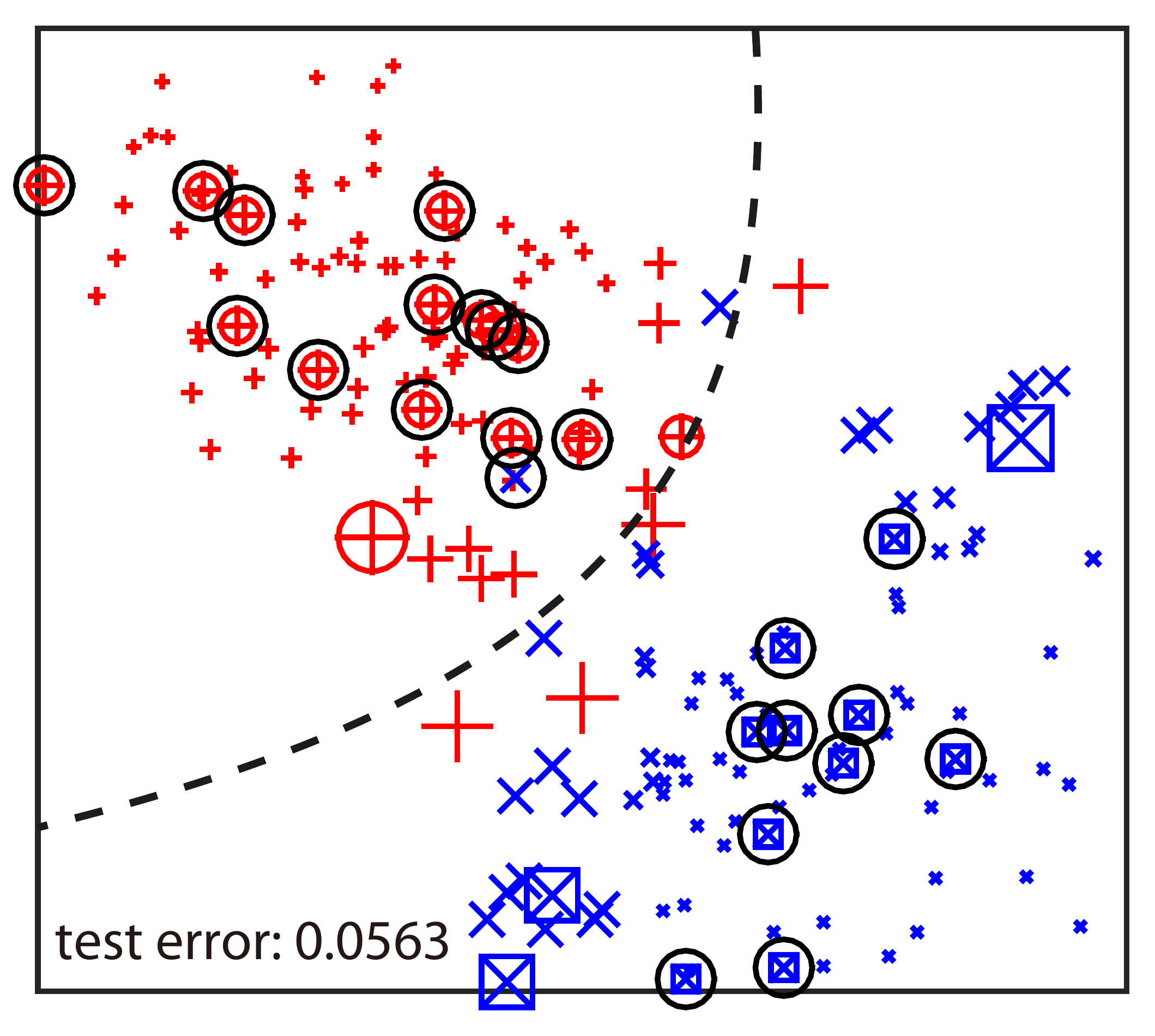}}
	%\hspace{1in}
	\subfigure[]{
		\label{experiment:legend}
		\includegraphics[width=0.45\linewidth]{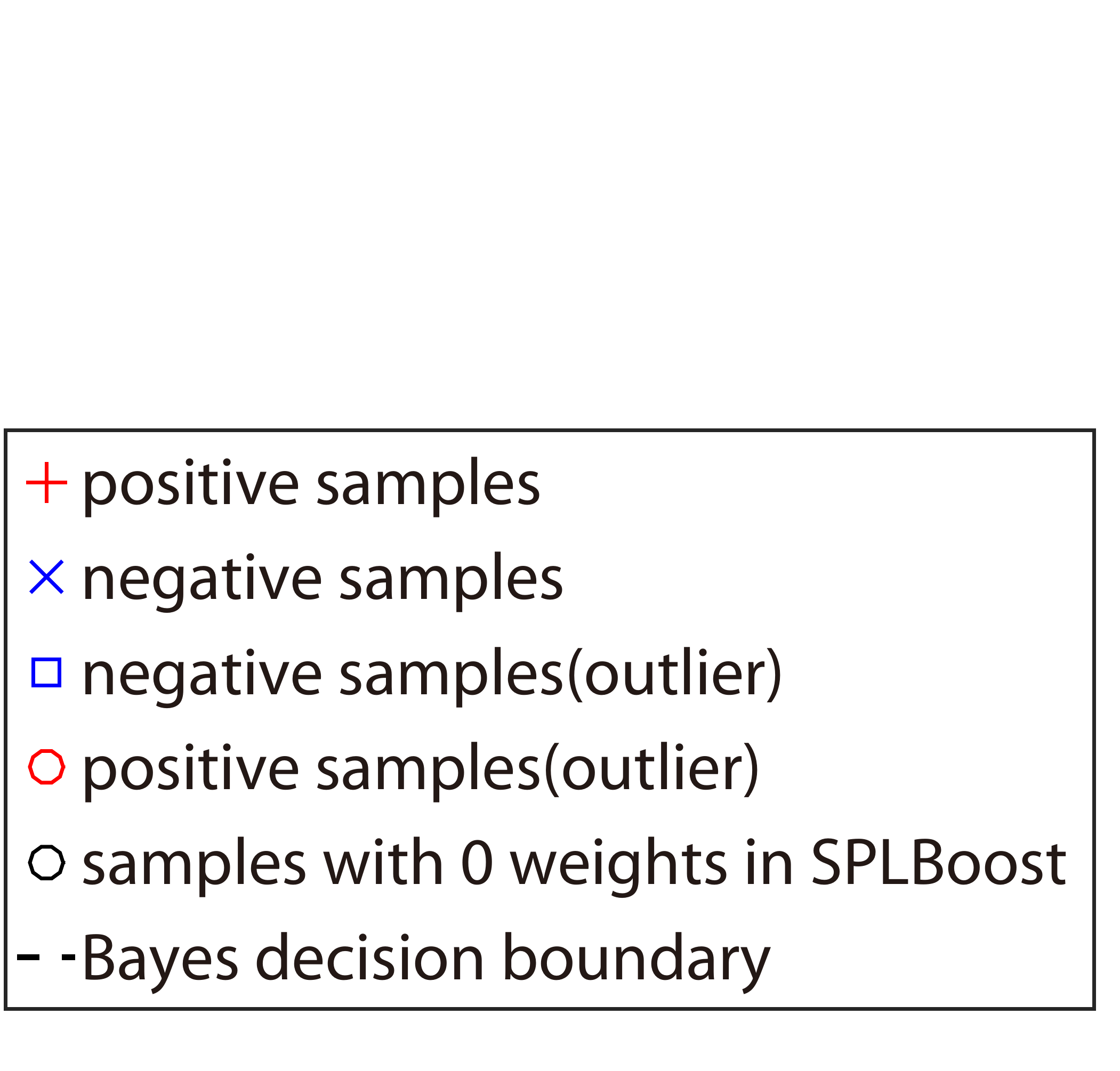}}
	\caption{The optimal Bayes decision boundary and the sample weights of AdaBoost, LogitBoost, SavageBoost, RBoost, RobustBoost and SPLBoost with $ \lambda=1.5, 2.0, 2.5 $, where the SP-regularizer in SPLBoost is Hard Weighting. 
	}\label{experiment:sample weight}
\end{figure}

Fig. \ref{experiment:sample weight} illustrates the performance of
all the competing boosting algorithms.  In all the sub-figures of Fig. \ref{experiment:sample weight}, pluses represent the positive training samples and cross marks represent the negative training samples. Blue squares represent those training samples that are generated from the Gaussian distribution of the negative class, but are labeled as positive samples. They are actually the outliers added to the negative class. In a similar way, red circles represent the outliers that added to the positive class. Black circles represent the training samples whose sample weights are 0 in SPLBoost, i.e., the samples which are considered to be outliers and thus have no effects on the weak learner training in SPLBoost algorithm. To visually observe the distribution of the sample weights of the various boosting algorithms, the sizes of the pluses, cross marks, blue squares and red circles are in proportion to the sample weights of the corresponding training samples. Since that the sample weights of the training samples marked by black circles are 0, the sizes of those black circles are all the same and have no relationship with their weights. As such, one can easily observe the training samples that the boosting algorithms focus on and the samples with 0 sample weights in SPLBoost. 

Basically, several observation can be made from Fig. \ref{experiment:sample weight}. Firstly, in Fig. \ref{experiment:AdaBoost}, most of the points with large sizes are surrounded by red circles or blue squares, which means that most of the training samples with large sample weights are outliers. This reveals that AdaBoost is so sensitive to the outliers. Actually, the exponential loss used by AdaBoost makes it to pay too much attention to the misclassified samples with large margins and those samples are usually outliers.  Secondly, it can be seen from Fig. \ref{experiment:LogitBoost} that the weights of most of the outliers are still larger than that of the correct training samples, though LogitBoost gives gentler weight on such outliers than AdaBoost. This is a common drawback of convex loss functions,  as stated by \cite{long2010random}.  Thirdly, Fig. \ref{experiment:SavageBoost}, Fig. \ref{experiment:RBoost} and Fig. \ref{experiment:RobustBoost} show the performances of SavageBoost, RBoost and RobustBoost respectively, which are more satisfactory than that of AdaBoost and LogitBoost because of their use of non-convex loss functions. In addition, compared with SaveBoost, RBoost makes the sizes of such same outliers are smaller due to the designing of numerical stably solver. It still assigns some relatively large weights to more number of points,  however, which can certainly affect its  performance in practice. As for RobustBoost, the weights of the outliers that near to Bayes decision boundary are still larger than that of other points, which means that RobustBoost is somehow affected by outliers. Fourthly,  based on Fig. 3(g-i), one can see that different $ \lambda $ gives different performance in SPLBoost. Specifically, Compared with Fig. \ref{experiment:SPLBoost with lambda=2}, in  Fig.  \ref{experiment:SPLBoost with lambda=1.5 } where $ \lambda $ is smaller, the algorithm is more "rigorous" to the outliers and more samples are considered to be outliers. On the contrary, in Fig. \ref{experiment:SPLBoost with lambda=2.5 } where $ \lambda $ is larger, the algorithm is more "tolerant" and fewer samples are considered to be outliers. In addition, it is not to see SPLBoost performs much better other competing methods. 

To sum up, compared with other boosting algorithms, SPLBoost gives more reasonable sample weight distribution and as a result can weaken the influences of the outliers in a larger extent.  

%In this subsection we compare the performances of various boosting algorithms in the synthetic Gaussian data set with outliers. From the results we can see that AdaBoost and LogitBoost which are based on convex loss functions are more sensitive to the outliers than the boosting algorithms based on non-convex losses, and compared with SavageBoost, RBoost and RobustBoost, SPLBoost has more reasonable sample weights distribution and can weaken the influences of the outliers in the largest extent.  

\subsection{UCI Data Set}
In this subsection, we shall demonstrate the experiment results in seventeen UCI data sets. And Table \ref{tab1} lists their detailed statistical information.

\begin{table}[t]
	\caption{Statistical Information of the UCI Data Sets}
	\center
	\begin{tabular}{lllllllllll}
		\hline
		Data Set & Sample Number & Feature Number & Percentage of \\ & & & Majority Class\\ \hline
		adult & 48842 & 14 & 75.9  \\ \hline
		bank & 45211  & 17 & 88.5  \\ \hline
		blood & 748 & 4  & 76.2  \\ \hline
		connect-4 & 67557 & 42  & 75.4  \\ \hline
		magic & 19020 & 10  & 64.8  \\ \hline
		miniboone & 130064 & 50  & 71.9  \\ \hline
		ozone & 2536 & 72  & 97.1  \\ \hline
		pima & 768 & 8  & 65.1  \\ \hline
		parkinsons & 195 & 22  & 75.4  \\ \hline
		pb-T-OR-D & 102 & 4  & 86.3  \\ \hline
		ringnorm & 7400 & 20  & 50.5  \\ \hline
		spambase & 4601 & 57  & 60.6  \\ \hline
		titanic & 2201 & 3  & 67.7  \\ \hline
		oocMerl4D & 1022 & 41  & 68.7  \\ \hline
		st-german-credit & 1000 & 24  & 70.0  \\ \hline
		twonorm & 7400 & 20 & 50.0 \\ \hline
		vc-2classes & 310& 6 & 67.7 \\ \hline
	\end{tabular}
	\label{tab1}
	%\vspace{-2mm}
\end{table} 

Our experiment settings are as follows. For every data set, we randomly select $ 70\% $ samples as the training data and the rest to be the test data. To evaluate the robustness of those compared boosting algorithms, we randomly select a certain proportion of the training data points and flip their labels. The noise levels are set to be $ 0\% $, $ 5\% $, $ 10\% $, $ 20\% $, $ 30\% $, respectively. The maximum iteration step is set to be 200, and a five fold cross validation procedure is used to determine the appropriate number of iteration steps of those boosting algorithms and some other parameters. As in the last experiment in synthetic Gaussian data set, the weak learner of  AdaBoost, SPLBoost and RobustBoost is chosen as C$ 4.5 $, while for LogitBoost and RBoost,  the regression tree CART is used as the weak learner. To illustrate the robust performance of SPLBoost with different SP-regularizers, we implement SPLBoost with four popular SP-regularizers namely Hard Weighting, linear Soft Weighting, Polynomial Soft Weighting with $ t=1.3$ and $t=4.0$, which are denoted as hard, linear, lowerconvex and upperconvex, respectively.  This procedure is repeated for 50 times and an average of the testing error rates that change with different noise levels for the various boosting algorithms are plotted in Fig.\ref{experiment:UCI data}.

\begin{figure*}[htbp]
	%\vspace{-4mm}
	%\centering
	\subfigure{
		\label{UCI experiment:adult}
		\includegraphics[width=0.225\linewidth]{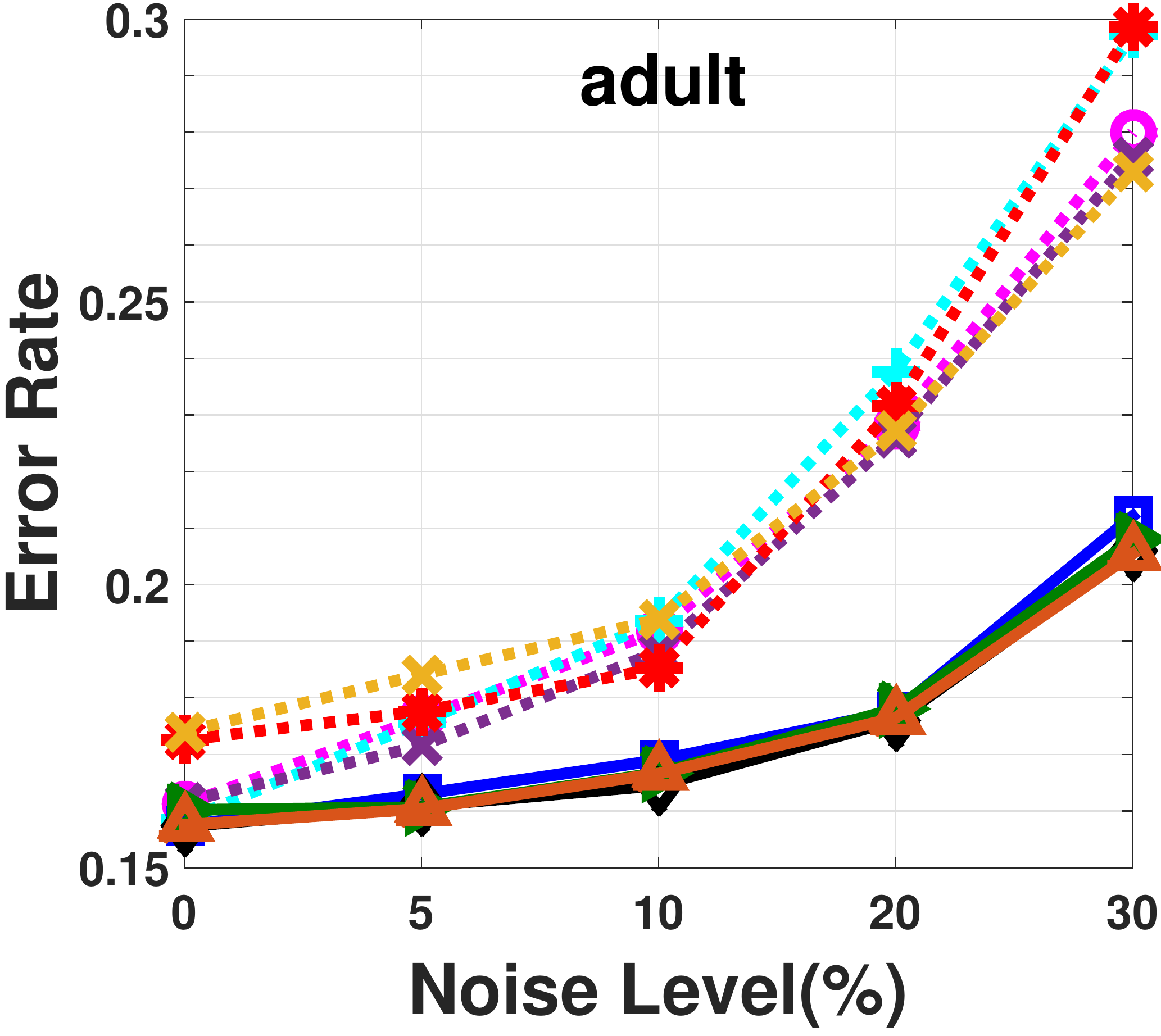}}
	%\hspace{1in}
	\subfigure{
		\label{UCI experiment:bank}
		\includegraphics[width=0.225\linewidth]{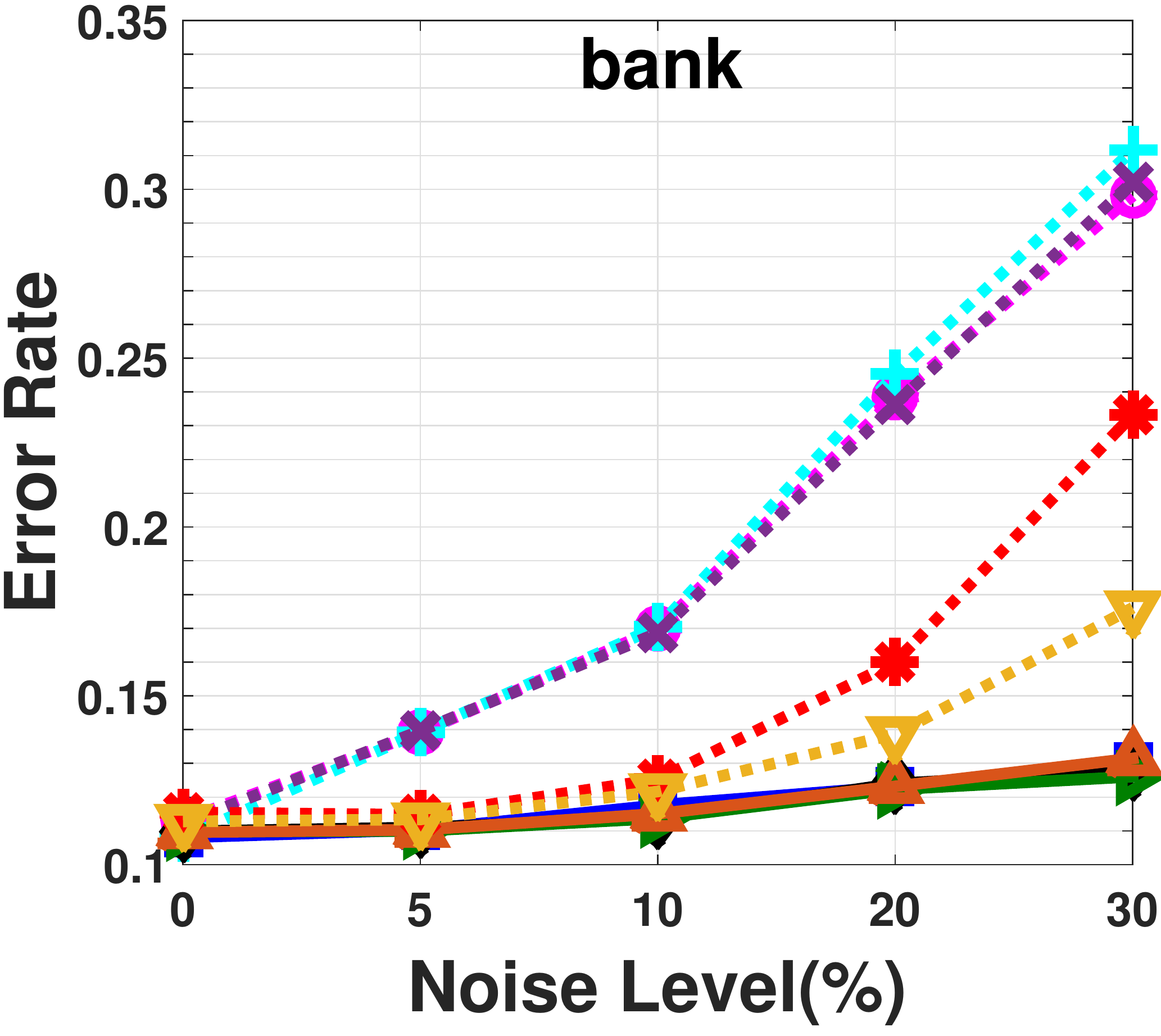}}
	\subfigure{
		\label{UCI experiment:blood}
		\includegraphics[width=0.225\linewidth]{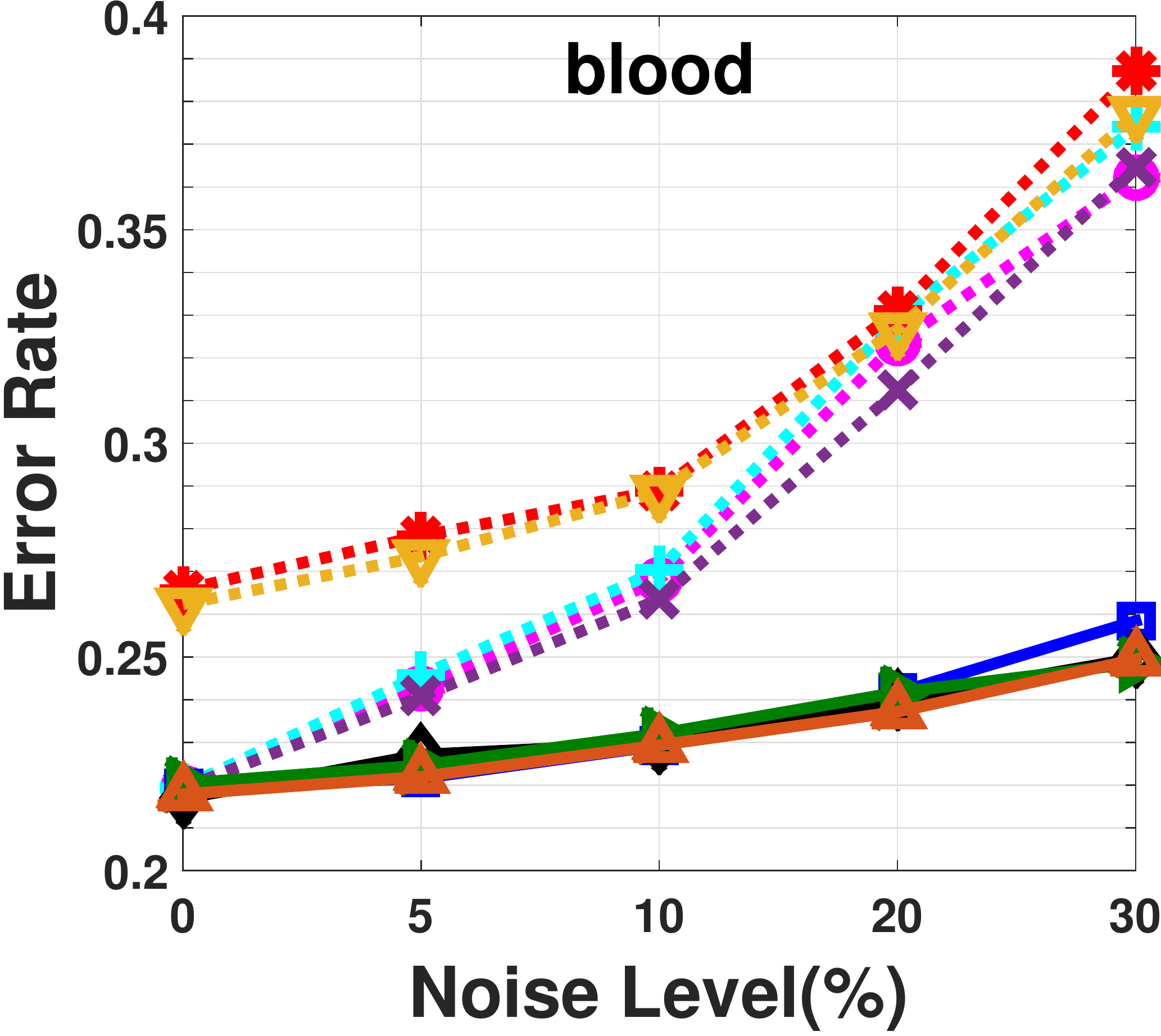}}
	%\hspace{1in}
	\subfigure{
		\label{UCI experiment:connect_4}
		\includegraphics[width=0.225\linewidth]{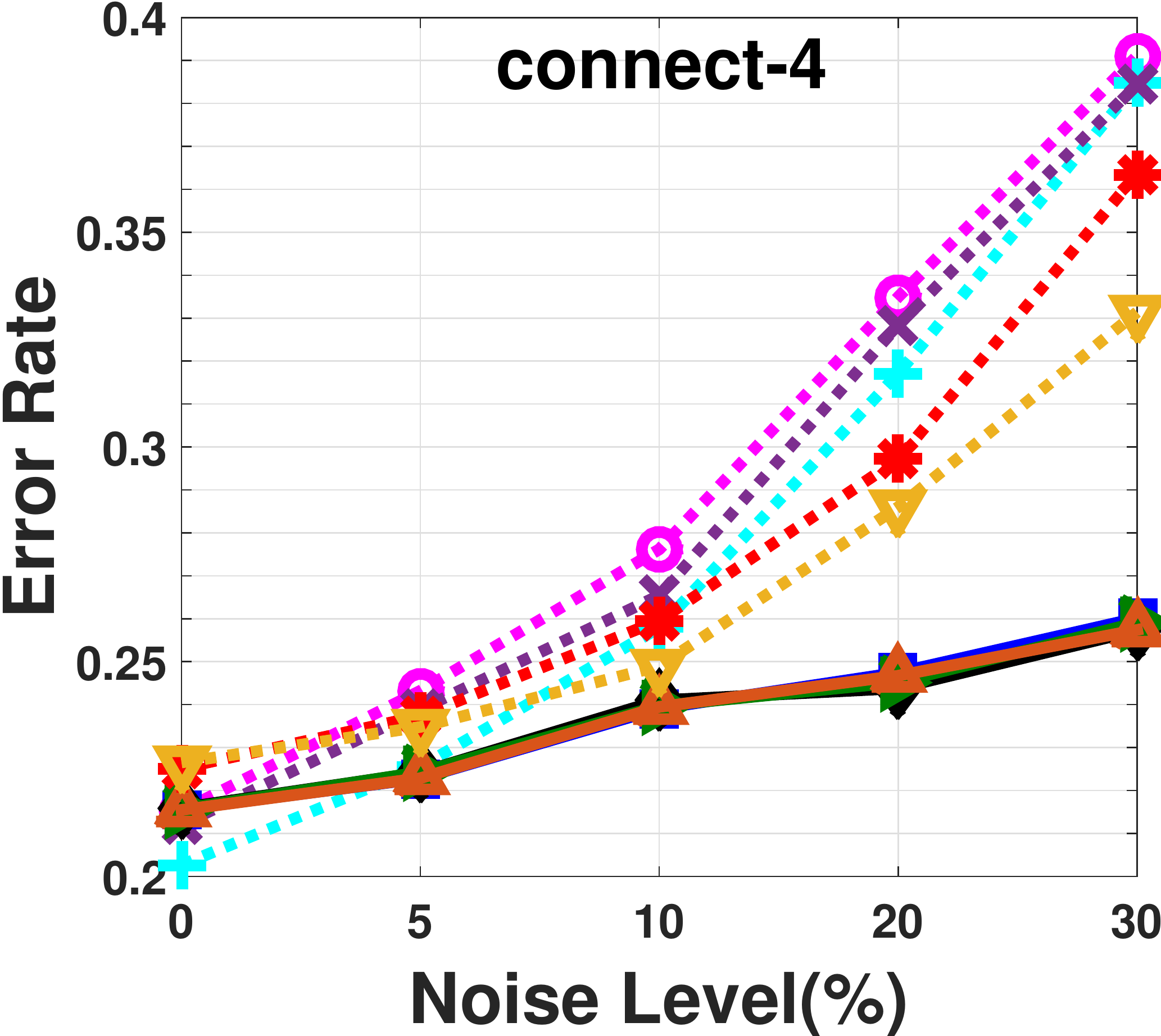}}\\
	\subfigure{
		\label{UCI experiment:magic}
		\includegraphics[width=0.225\linewidth]{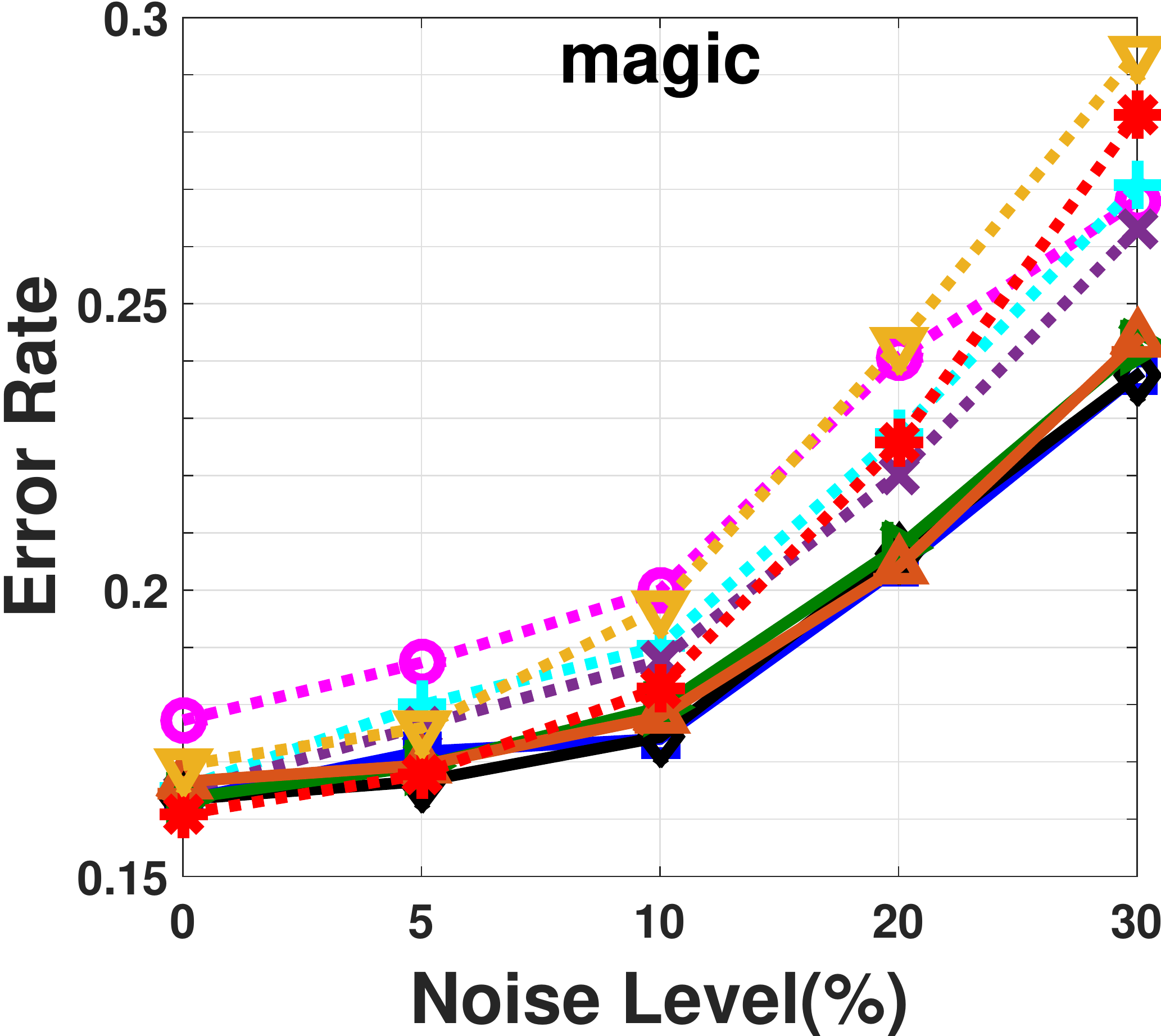}}
	%\hspace{1in}
	\subfigure{
		\label{UCI experiment:miniboone}
		\includegraphics[width=0.225\linewidth]{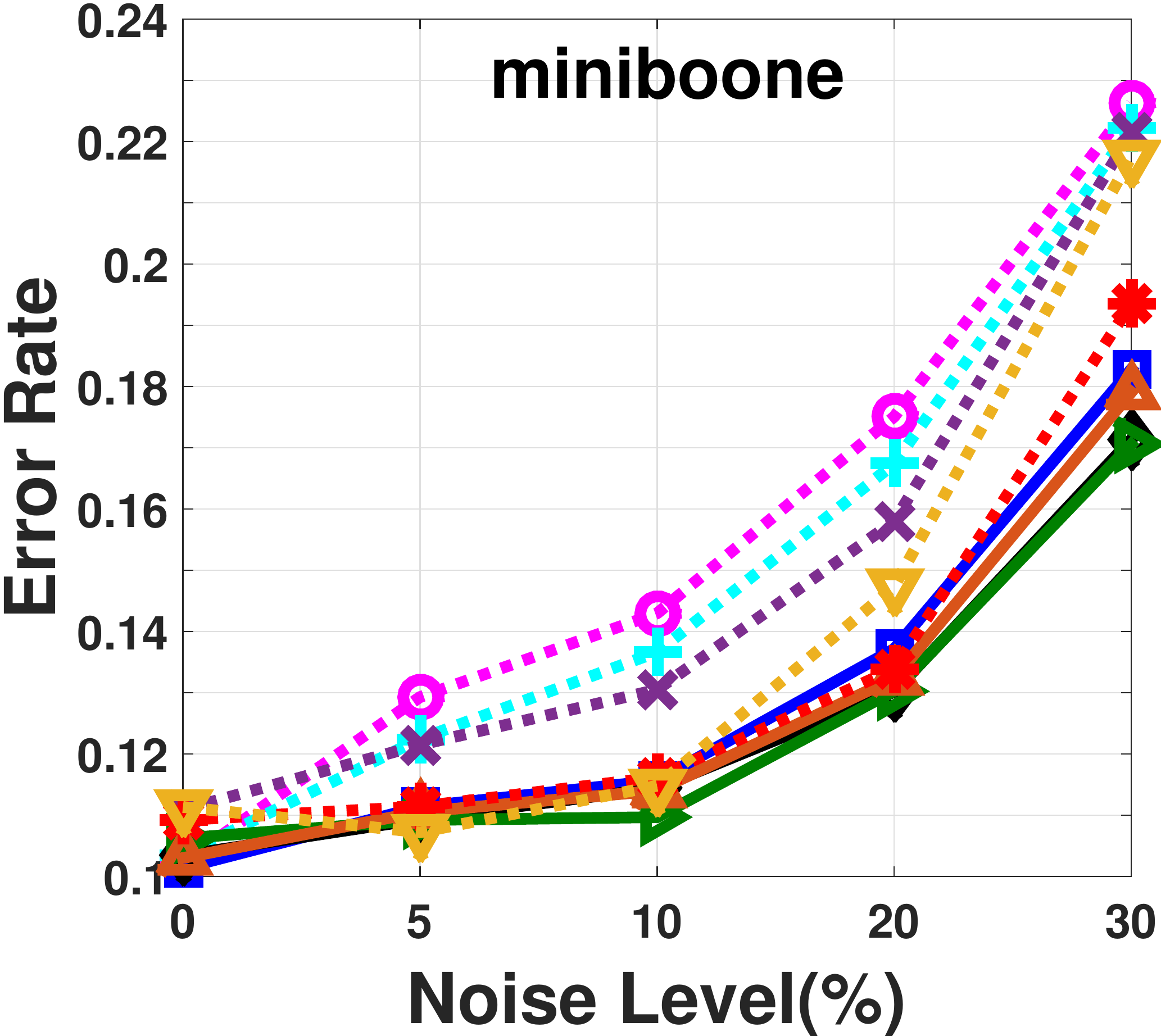}}
	\subfigure{
		\label{UCI experiment:oocMerl4D }
		\includegraphics[width=0.225\linewidth]{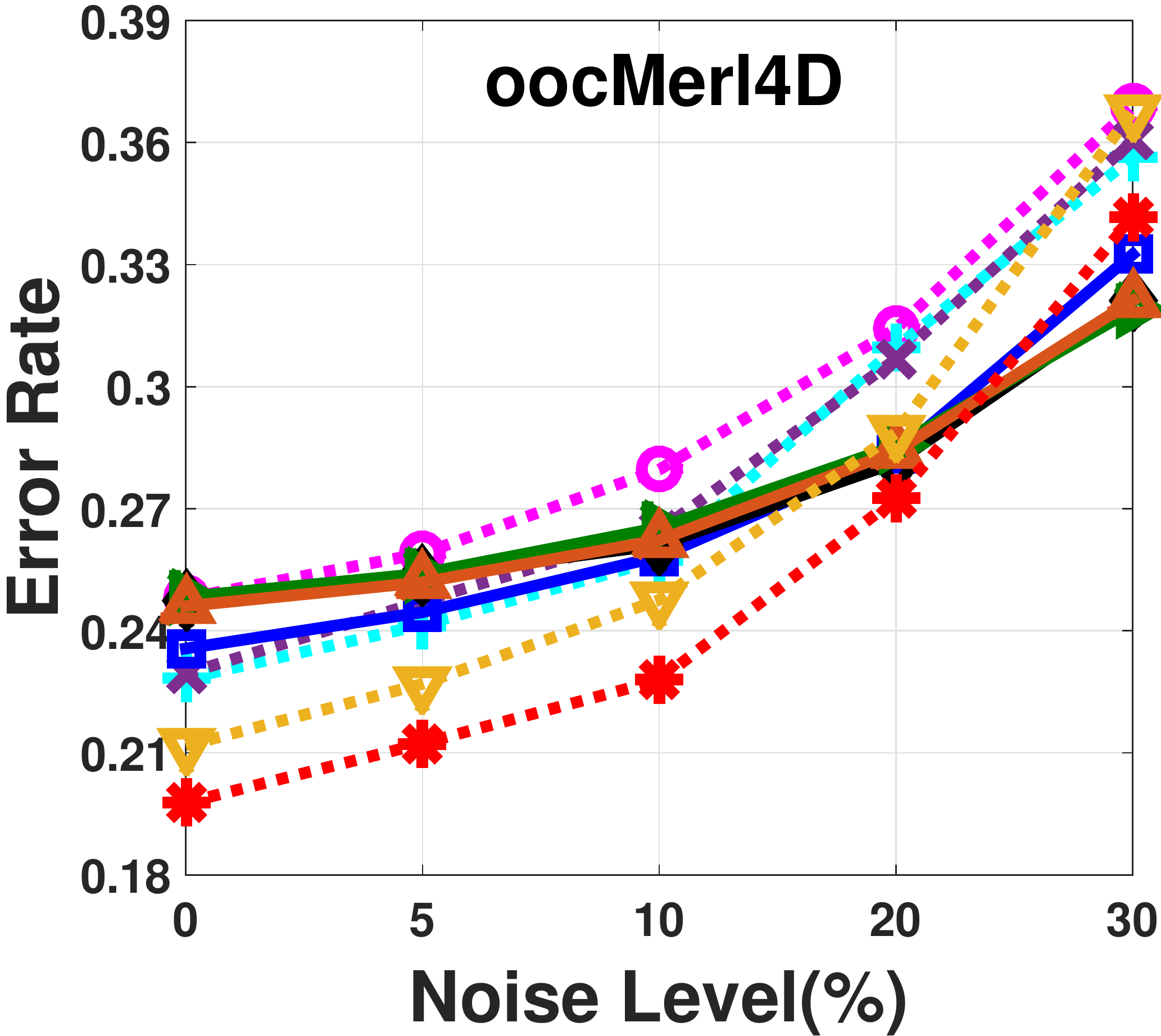}}
	%\hspace{1in}
	\subfigure{
		\label{UCI experiment:ozone}
		\includegraphics[width=0.225\linewidth]{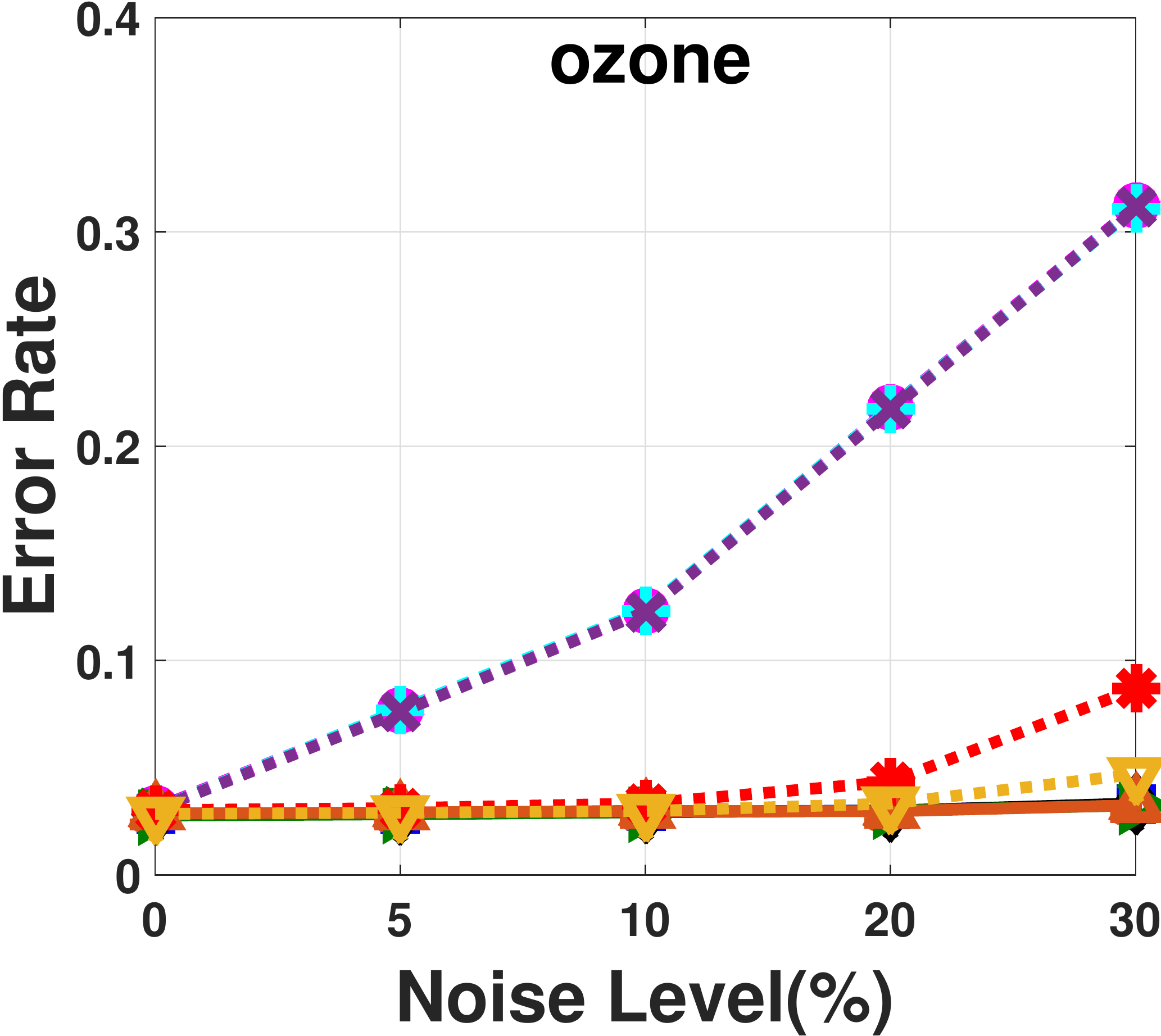}}\\
	\subfigure{
		\label{UCI experiment:parkinsons }
		\includegraphics[width=0.225\linewidth]{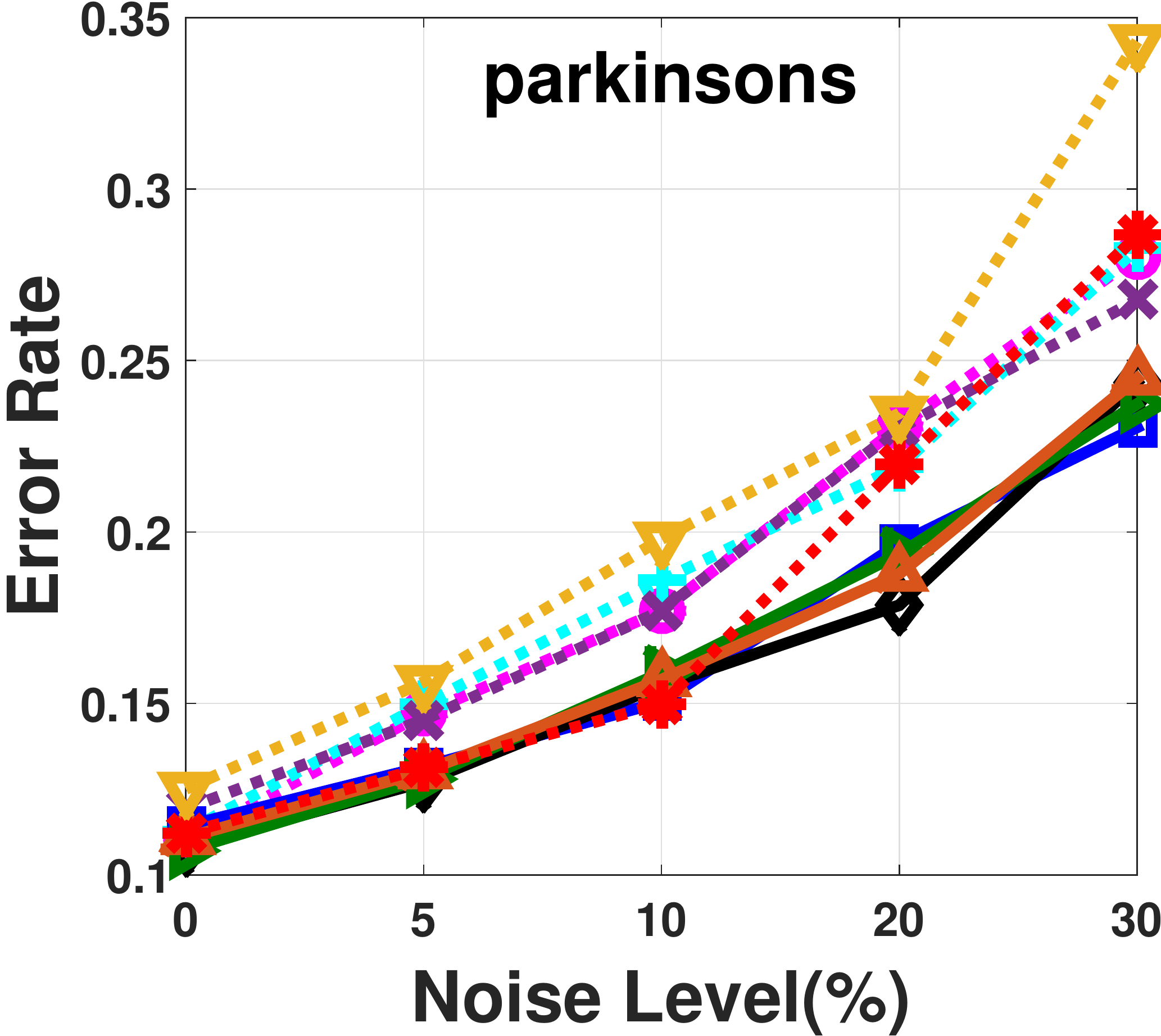}}
	%\hspace{1in}
	\subfigure{
		\label{UCI experiment:pb_T_OR_D}
		\includegraphics[width=0.225\linewidth]{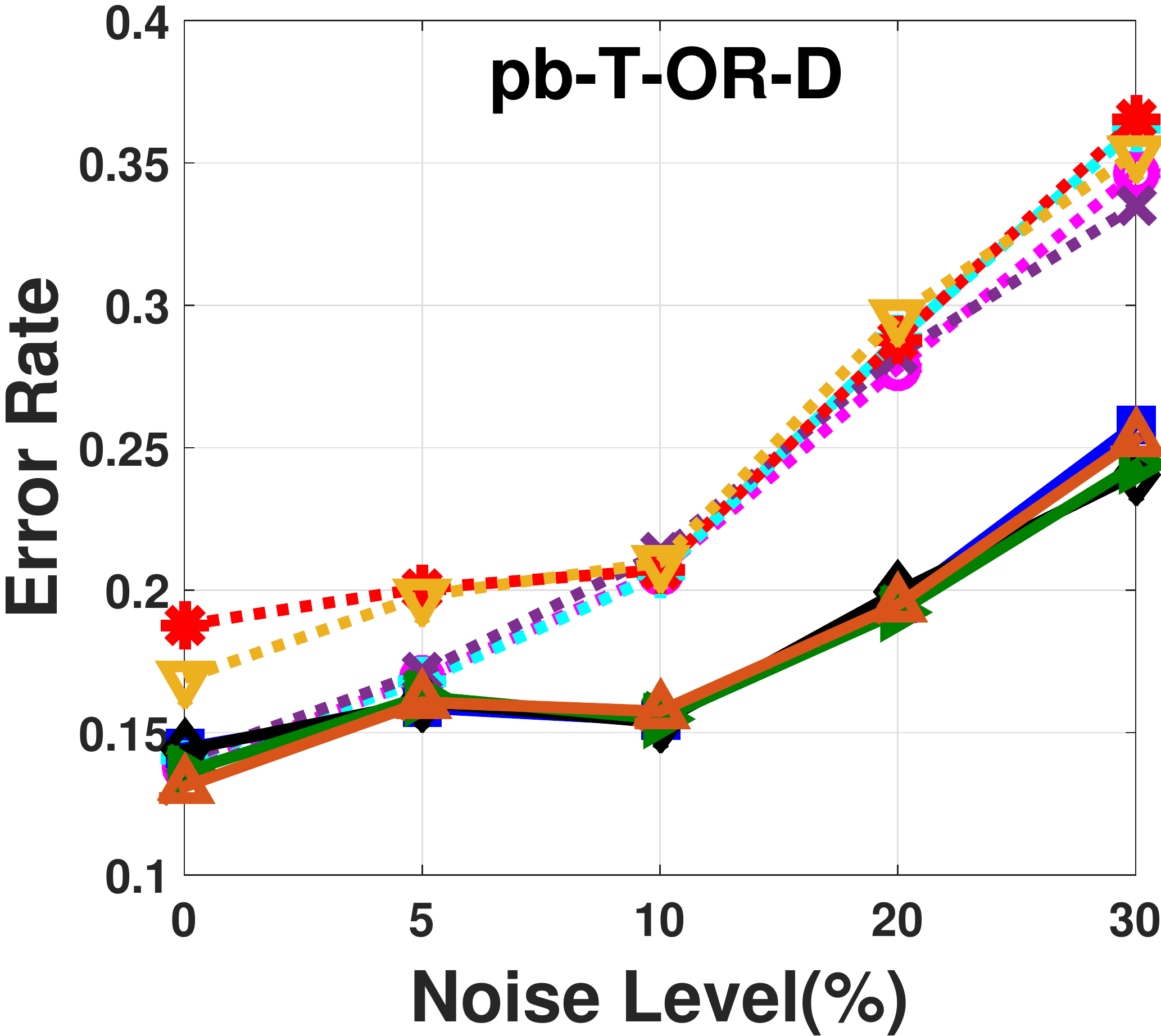}}
	\subfigure{
		\label{UCI experiment:pima2}
		\includegraphics[width=0.225\linewidth]{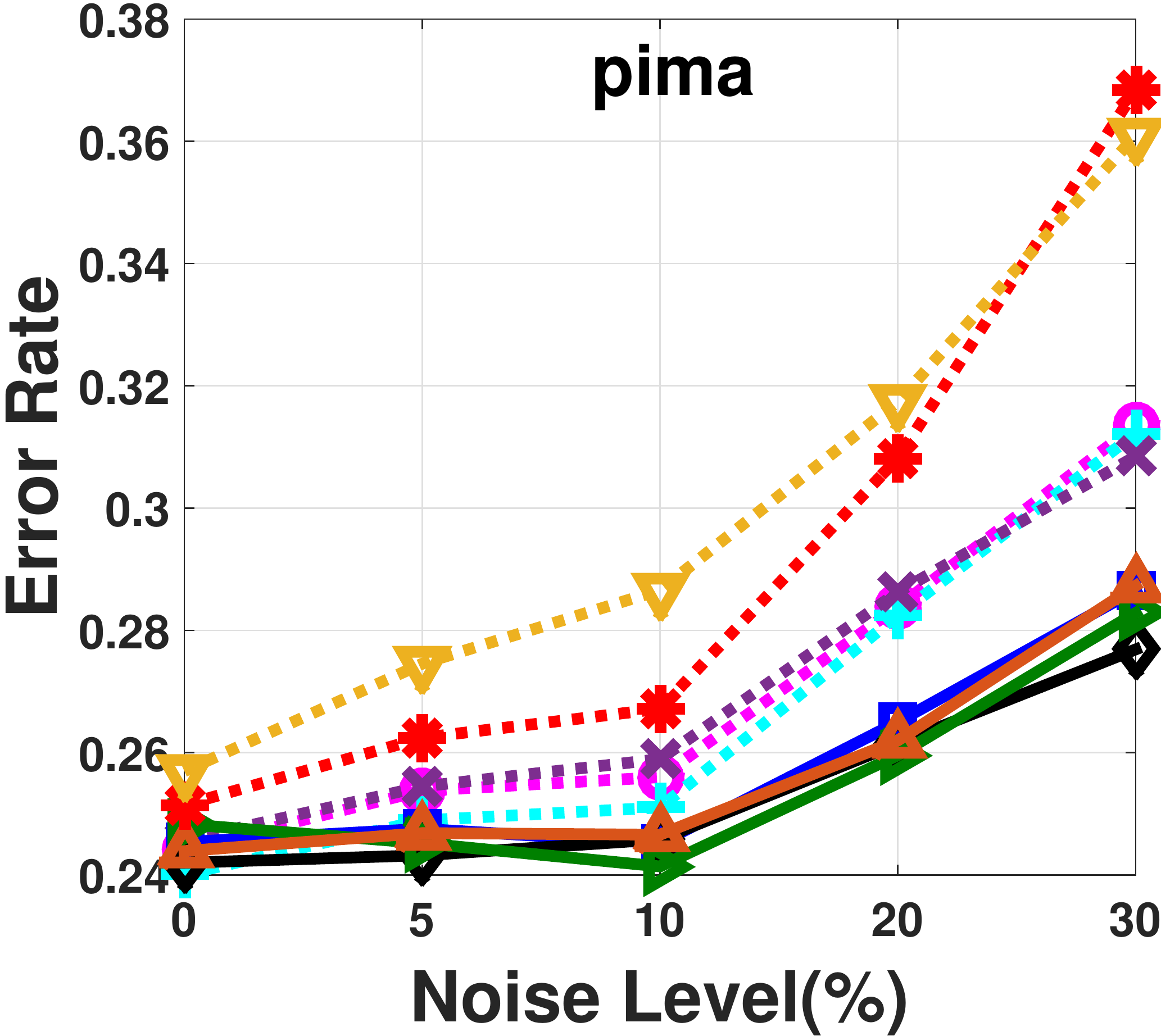}}
	\subfigure{
		\label{UCI experiment:ringnorm}
		\includegraphics[width=0.225\linewidth]{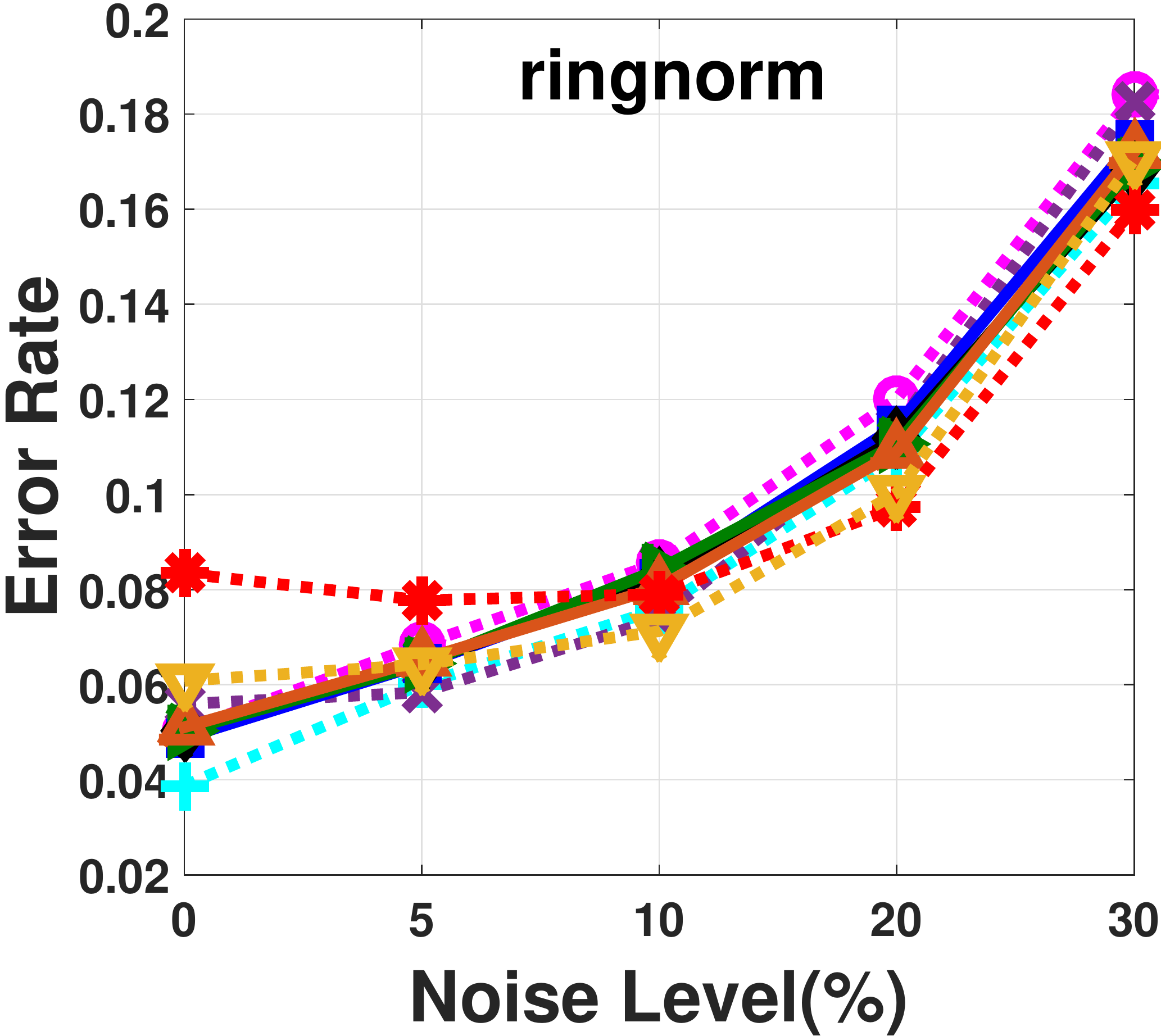}}\\
	\subfigure{
		\label{UCI experiment:spambase}
		\includegraphics[width=0.225\linewidth]{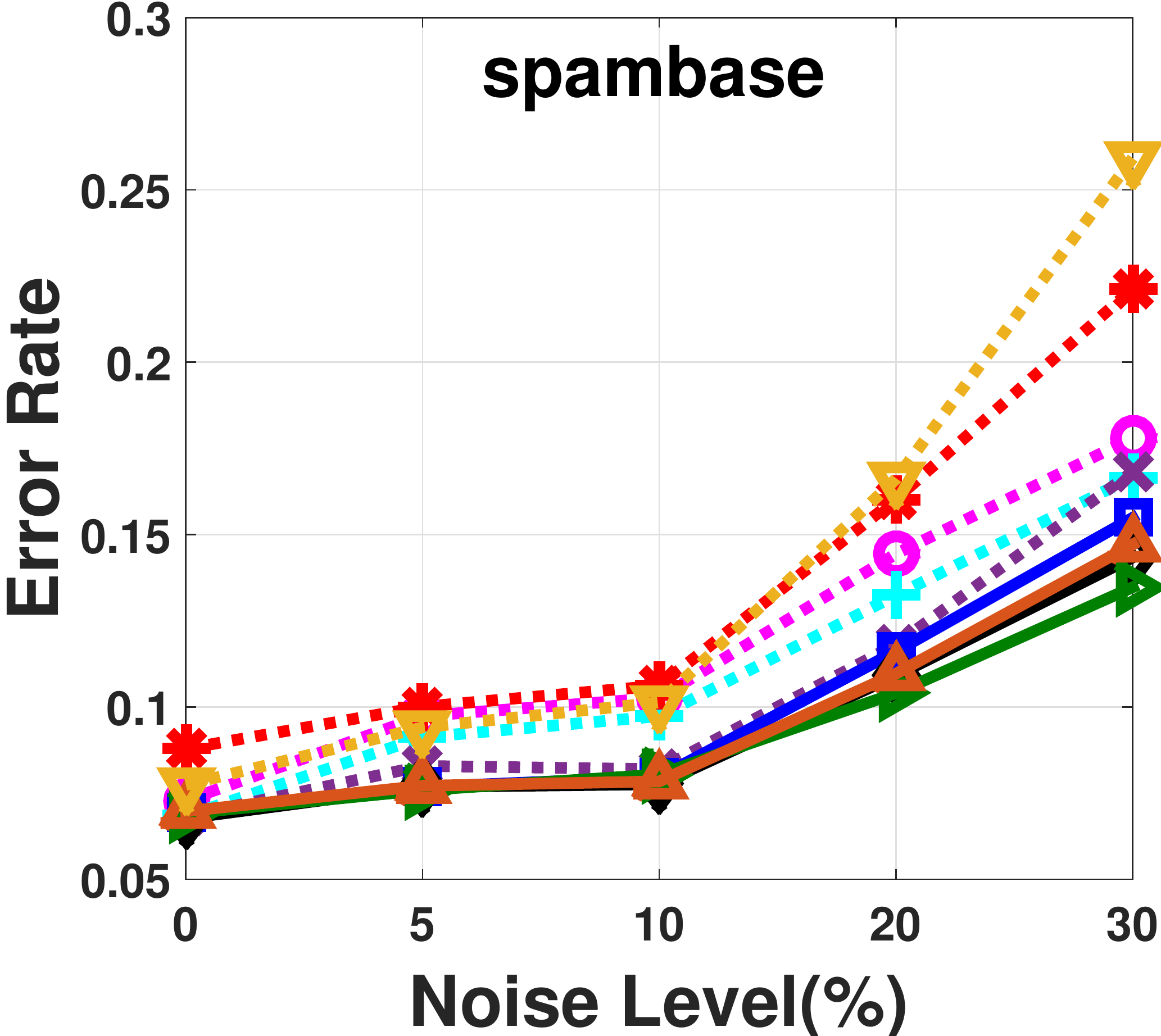}}
	\subfigure{
		\label{UCI experiment:statlog}
		\includegraphics[width=0.225\linewidth]{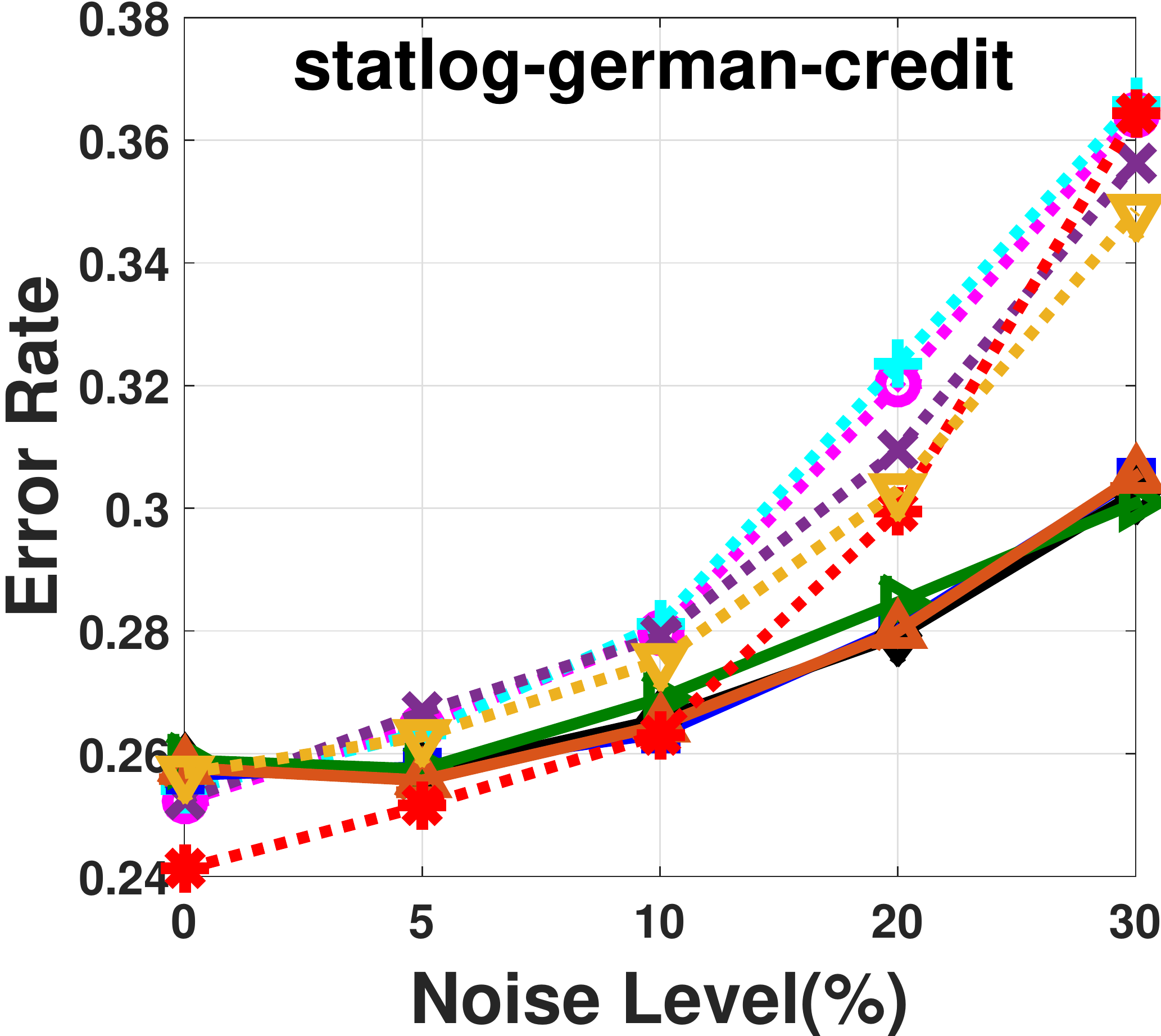}}
	\subfigure{
		\label{UCI experiment:titanic}
		\includegraphics[width=0.225\linewidth]{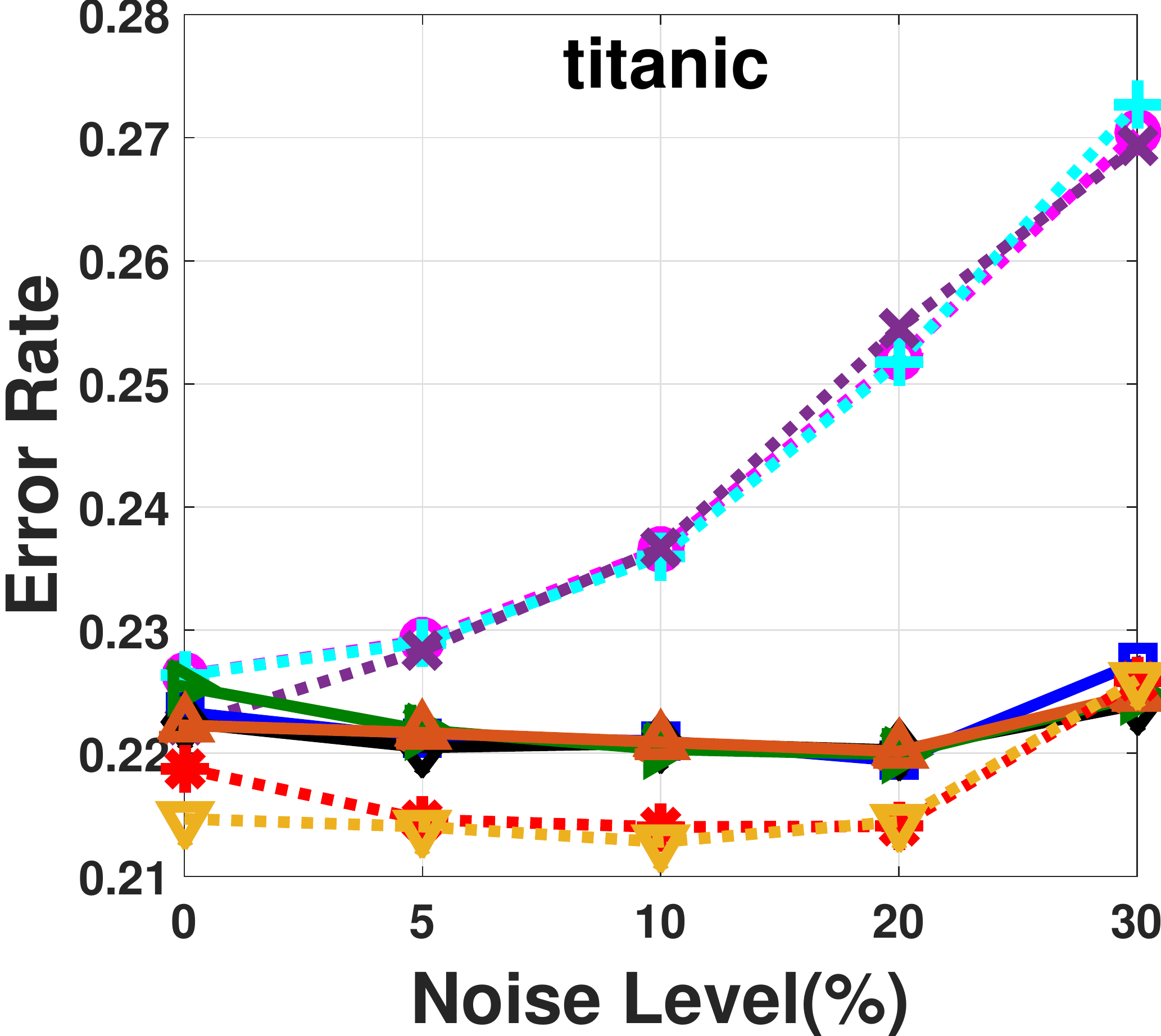}}
	\subfigure{
		\label{UCI experiment:twonorm}
		\includegraphics[width=0.225\linewidth]{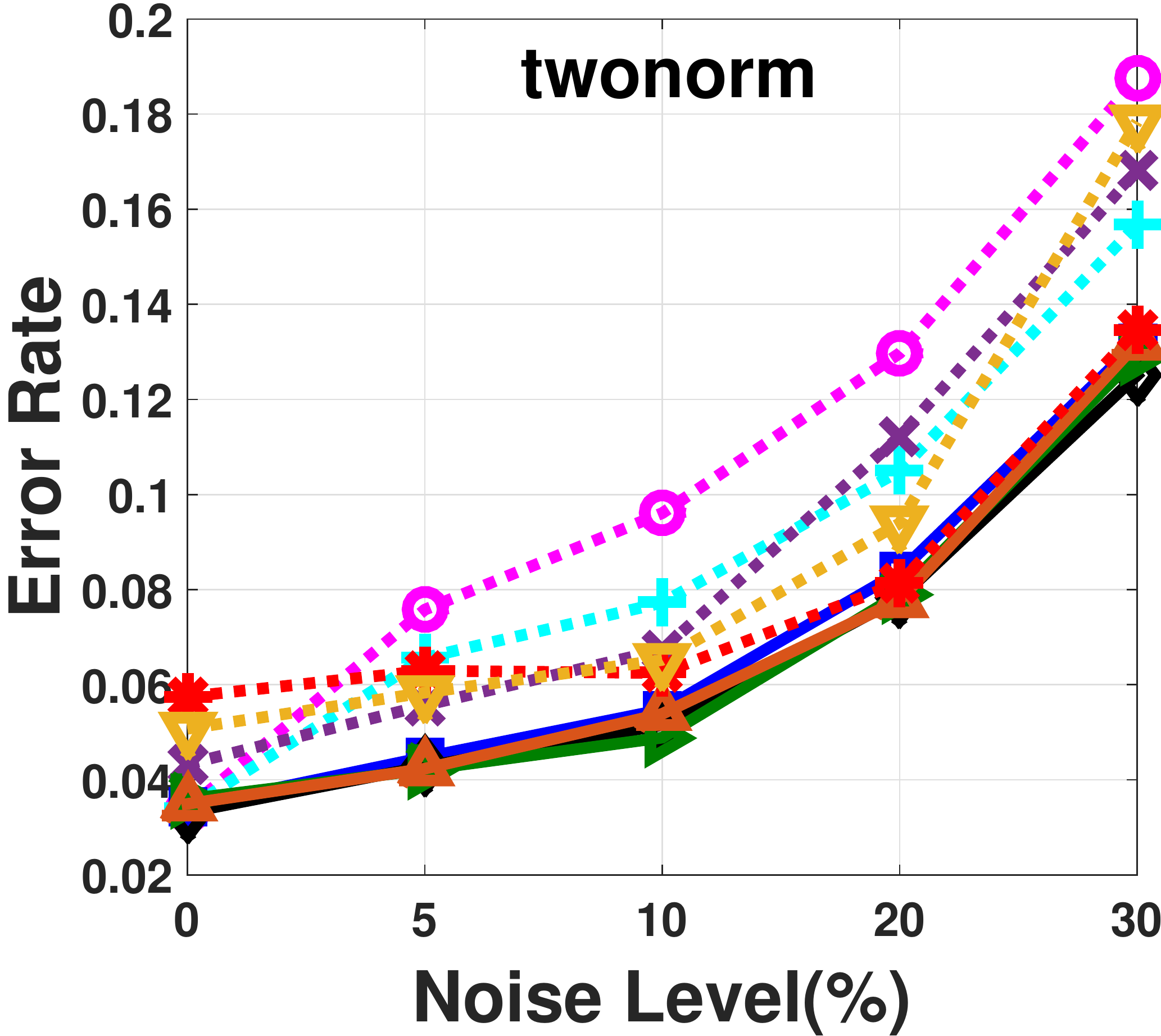}}\\
	\subfigure{
		\label{UCI experiment:vertebral}
		\includegraphics[width=0.225\linewidth]{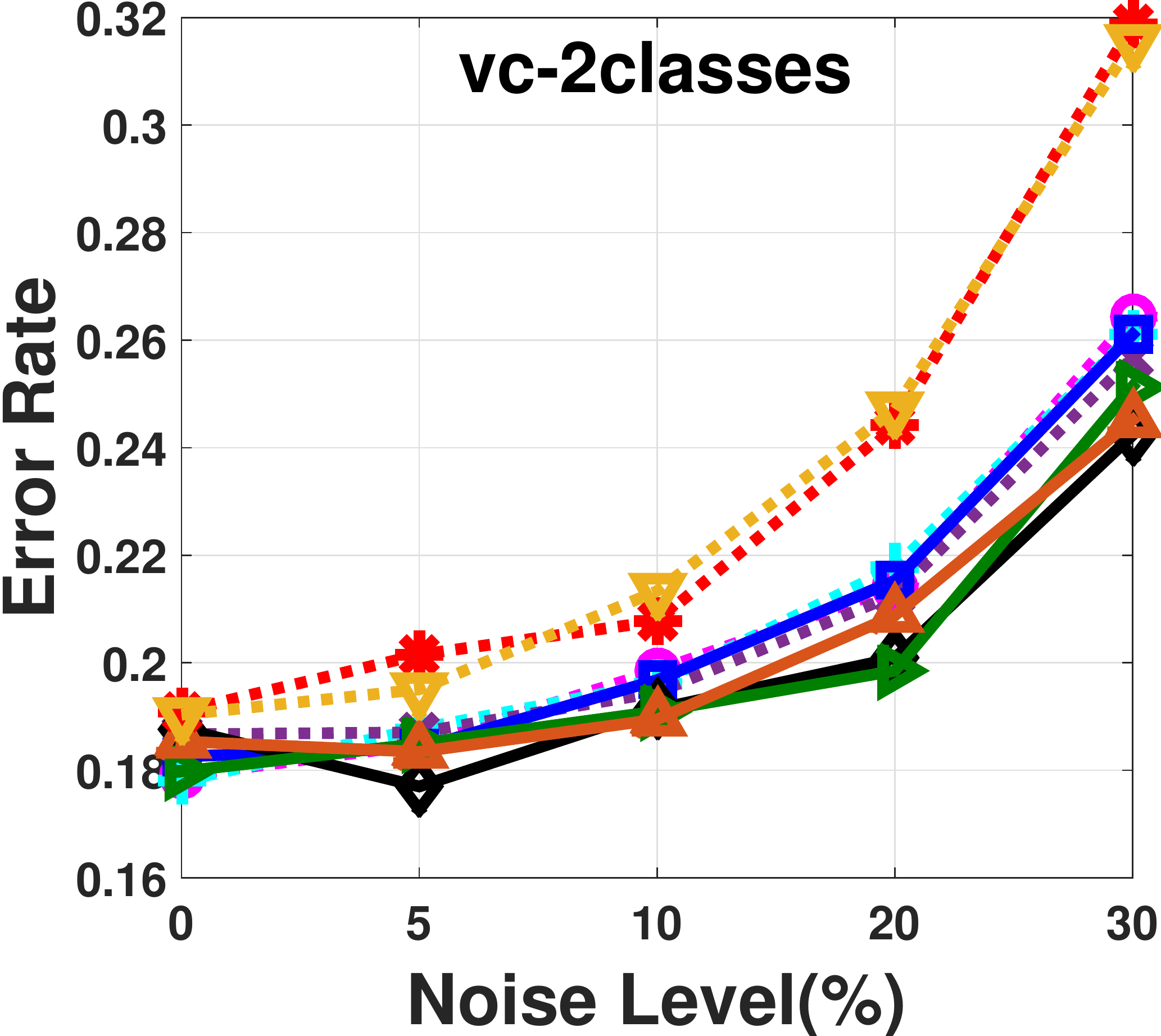}}
	\subfigure{
		\label{UCI experiment:legend2}
		\includegraphics[width=0.225\linewidth]{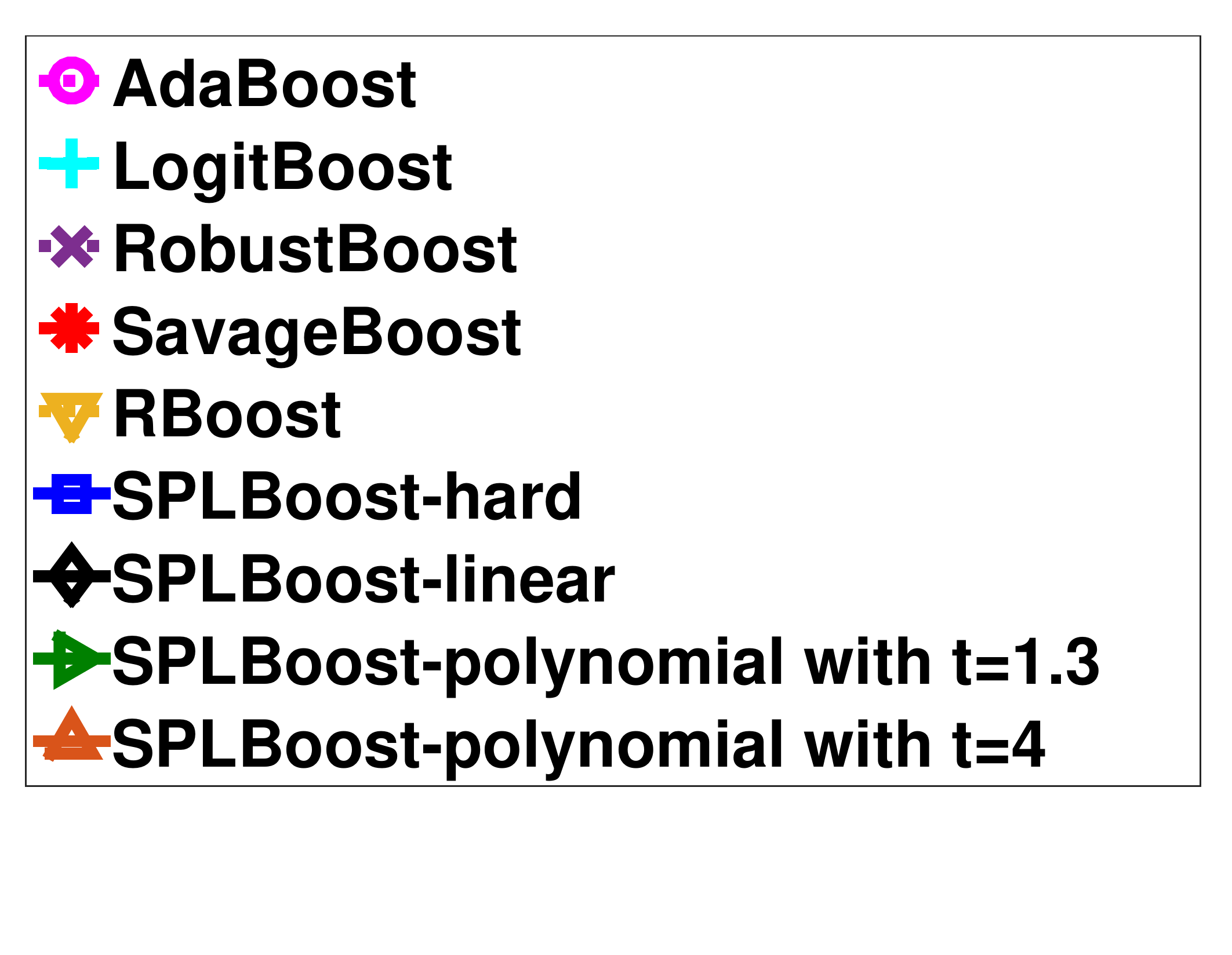}}
	\caption{The testing error rates changing with different noise levels for the various boosting algorithms on seventeen UCI data sets. 
	}\label{experiment:UCI data}
\end{figure*}
\begin{figure*}[htbp] 
	\centering 
	\includegraphics[width=0.9\linewidth]{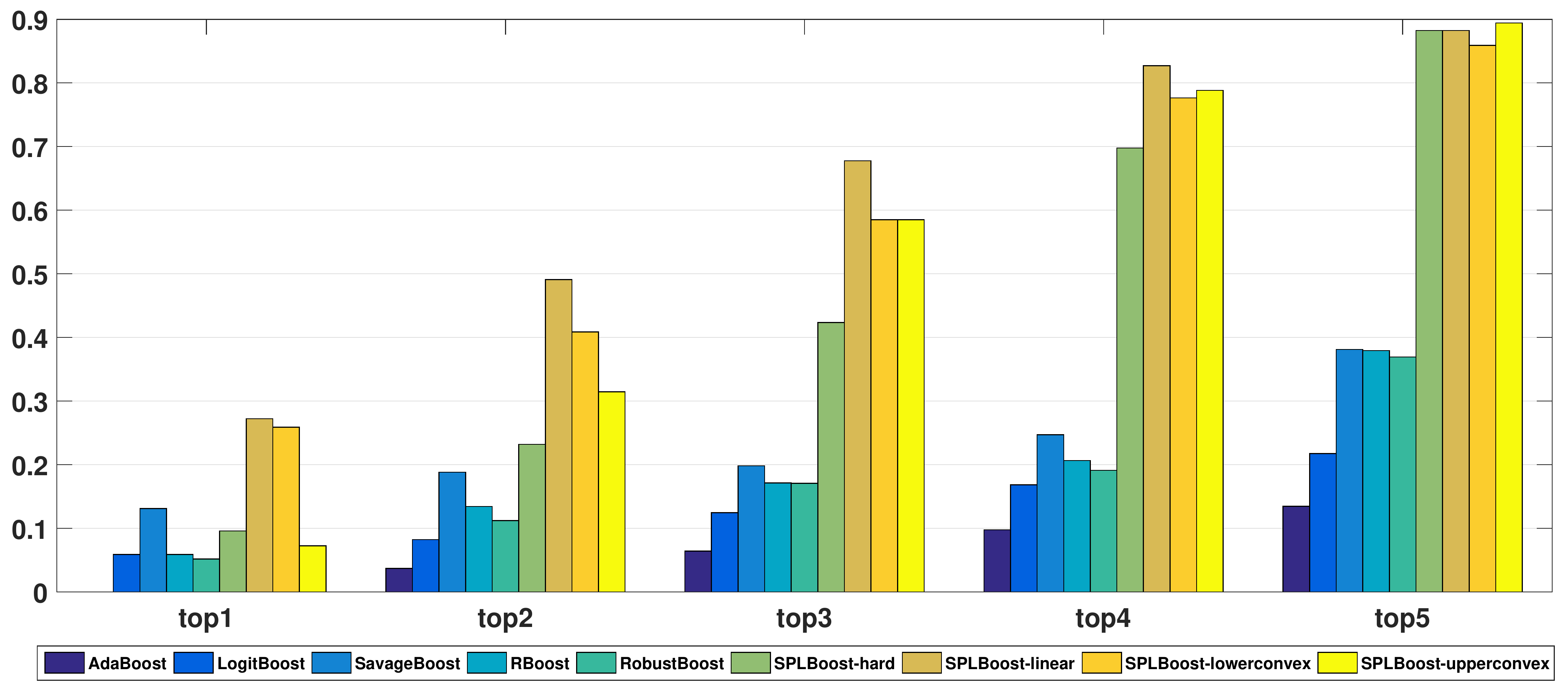}
	\caption{Statistics of rank values on seventeen UCI data sets with five noise levels. 
	}\label{experiment:UCI data statistics}
\end{figure*}

It can be seen from Fig.\ref{experiment:UCI data} that once there are outliers in training data, the performance of AdaBoost can be heavily depraved and is not comparable with LogitBoost, SavageBoost, RBoost, RobustBoost and SPLBoost, which confirms that AdaBoost is very sensitive to the noisy data. We can also see that for most of the data sets, SPLBoost  gives lower test errors than other boosting algorithms, which reveals that SPLBoost has best robustness among all the compared methods. Additionally, it is not hard to see that there are no significantly difference between the performance of the SPLBoost using four different SP-regularizers. 

To better demonstrate the performance of those compared algorithms,  we can rank their performance from 1 (the algorithm with the lowest mean testing error rate) to 9 (the algorithm with the highest mean testing error rate) in those total 85 cases (17 UCI data sets and 5 noise levels). Then we calculates the ratio of the data sets for each boosting algorithm among the top-$n$ ranked ones and the results are summarized in Fig. \ref{experiment:UCI data statistics}. It can be easily seen that in most cases, the performance of SPLBoost is much better than other competing boosting algorithms, which clearly confirms that SPLBoost has better resistance to large noise and outliers. 

\section{Conclusion}
Boosting can be interpreted as a gradient descent technique to minimize an underlying loss function, and the loss function determines the robustness of the algorithm. Convex loss functions such as the exponential loss used by AdaBoost and the logistic loss used by LogitBoost have been proven to be sensitive to the outliers. Non-convex loss functions such as Savage loss used by SavageBoost and Savage2 loss used by RBoost have illustrated  superior robustness over popular convex losses, however, solving the non-convex optimization problem derived from non-convex losses is not an easy task.

In this paper, instead of designing new loss function, we combine the classical Discrete AdaBoost algorithm with self-paced learning regime, that is, a robust algorithm framework which has been attracting troumendous attention in machine learning and computer vision. Thus,  we come up with a new robust boosting algorithm named SPLBoost. Experiments shows that SPLBoost could have a superior  performance over other popular ones once outliers exist in training data.

However, there are still some interesting works need to be done in the future. On one hand, it is not hard to see that the SPLBoost can be treated as a general framework to improve the robustness of various boosting algorithms besides AdaBoost. As such,  one can try some other popular boosting algorithms such as LogitBoost, $L_2$Boost to get better performance. On the other hand, although we have proven the equivalence between the SPLBoost and a MM algorithm implemented on a latent non-convex objective function, the detailed theoretical properties of SPLBoost, including consistency,
convergence rate, and error bound, are needed to be further investigated.  

\section*{Acknowledgment}

This work was supported by the National Natural Science Foundation of China (Grant Nos. 11501440, 61303168, 61333019 and 61373114).

% or

\small
\bibliographystyle{unsrt}
\bibliography{SPLBoost_ref_final}

\begin{thebibliography}{10}

\bibitem{freund1995desicion}
Yoav Freund and Robert~E Schapire.
\newblock A desicion-theoretic generalization of on-line learning and an
  application to boosting.
\newblock In {\em European Conference on Computational Learning Theory}, pages
  23--37. Springer, 1995.

\bibitem{schapire2012boosting}
Robert~E Schapire and Yoav Freund.
\newblock {\em Boosting: Foundations and Algorithms}.
\newblock MIT press, 2012.

\bibitem{meir2003introduction}
Ron Meir and Gunnar R{\"a}tsch.
\newblock An introduction to boosting and leveraging.
\newblock In {\em Advanced Lectures on Machine Learning}, pages 118--183.
  Springer, 2003.

\bibitem{tsao2007stochastic}
C~Andy Tsao and Yuan-chin~Ivan Chang.
\newblock A stochastic approximation view of boosting.
\newblock {\em Computational Statistics \& Data Analysis}, 52(1):325--334,
  2007.

\bibitem{friedman2001greedy}
Jerome~H Friedman.
\newblock Greedy function approximation: a gradient boosting machine.
\newblock {\em The Annals of statistics}, 29(5):1189--1232, 2001.

\bibitem{freund1996experiments}
Yoav Freund, Robert~E Schapire, et~al.
\newblock Experiments with a new boosting algorithm.
\newblock In {\em International Conference on Machine Learning}, volume~96,
  pages 148--156, 1996.

\bibitem{schapire1999improved}
Robert~E Schapire and Yoram Singer.
\newblock Improved boosting algorithms using confidence-rated predictions.
\newblock {\em Machine Learning}, 37(3):297--336, 1999.

\bibitem{viola2001rapid}
Paul Viola and Michael Jones.
\newblock Rapid object detection using a boosted cascade of simple features.
\newblock In {\em Proceeding of the IEEE Conference on Computer Vision and
  Pattern Recognition}, pages 511--518. IEEE, 2001.

\bibitem{viola2004robust}
Paul Viola and Michael~J Jones.
\newblock Robust real-time face detection.
\newblock {\em International Journal of Computer Vision}, 57(2):137--154, 2004.

\bibitem{lv2006recognition}
Fengjun Lv and Ramakant Nevatia.
\newblock Recognition and segmentation of 3-d human action using hmm and
  multi-class adaboost.
\newblock In {\em European Conference on Computer Vision}, pages 359--372.
  Springer, 2006.

\bibitem{bergstra2006aggregate}
James Bergstra, Norman Casagrande, Dumitru Erhan, Douglas Eck, and Bal{\'a}zs
  K{\'e}gl.
\newblock Aggregate features and adaboost for music classification.
\newblock {\em Machine Learning}, 65(2-3):473--484, 2006.

\bibitem{chan2008evaluation}
Jonathan Cheung-Wai Chan and Desir{\'e} Paelinckx.
\newblock Evaluation of random forest and adaboost tree-based ensemble
  classification and spectral band selection for ecotope mapping using airborne
  hyperspectral imagery.
\newblock {\em Remote Sensing of Environment}, 112(6):2999--3011, 2008.

\bibitem{frenay2014classification}
Beno{\^\i}t Fr{\'e}nay and Michel Verleysen.
\newblock Classification in the presence of label noise: a survey.
\newblock {\em IEEE Transactions on Neural Networks and Learning Systems},
  25(5):845--869, 2014.

\bibitem{khoshgoftaar2011comparing}
Taghi~M Khoshgoftaar, Jason Van~Hulse, and Amri Napolitano.
\newblock Comparing boosting and bagging techniques with noisy and imbalanced
  data.
\newblock {\em IEEE Transactions on Systems, Man, and Cybernetics-Part A:
  Systems and Humans}, 41(3):552--568, 2011.

\bibitem{cao2012noise}
Jingjing Cao, Sam Kwong, and Ran Wang.
\newblock A noise-detection based adaboost algorithm for mislabeled data.
\newblock {\em Pattern Recognition}, 45(12):4451--4465, 2012.

\bibitem{brodley1999identifying}
Carla~E. Brodley and Mark~A. Friedl.
\newblock Identifying mislabeled training data.
\newblock {\em Journal of Artificial Intelligence Research}, 11:131--167, 1999.

\bibitem{friedman2000additive}
Jerome Friedman, Trevor Hastie, Robert Tibshirani, et~al.
\newblock Additive logistic regression: a statistical view of boosting (with
  discussion and a rejoinder by the authors).
\newblock {\em The Annals of Statistics}, 28(2):337--407, 2000.

\bibitem{shen2013fully}
Chunhua Shen, Hanxi Li, and Anton Van Den~Hengel.
\newblock Fully corrective boosting with arbitrary loss and regularization.
\newblock {\em Neural Networks}, 48:44--58, 2013.

\bibitem{domingo2000madaboost}
Carlos Domingo and Osamu Watanabe.
\newblock Madaboost: A modification of adaboost.
\newblock In {\em Proceeding of the Thirteenth Annual Conference on
  Computational Learning Theory}, pages 180--189. Springer, 2000.

\bibitem{collins2002logistic}
Michael Collins, Robert~E Schapire, and Yoram Singer.
\newblock Logistic regression, adaboost and bregman distances.
\newblock {\em Machine Learning}, 48(1):253--285, 2002.

\bibitem{li2016boosting}
Alexander~Hanbo Li and Jelena Bradic.
\newblock Boosting in the presence of outliers: adaptive classification with
  non-convex loss functions.
\newblock {\em To appear in Journal of the American Statistical Association},
  2016.

\bibitem{koltchinskii2002empirical}
Vladimir Koltchinskii and Dmitry Panchenko.
\newblock Empirical margin distributions and bounding the generalization error
  of combined classifiers.
\newblock {\em The Annals of Statistics}, pages 1--50, 2002.

\bibitem{zhang2005boosting}
Tong Zhang, Bin Yu, et~al.
\newblock Boosting with early stopping: convergence and consistency.
\newblock {\em The Annals of Statistics}, 33(4):1538--1579, 2005.

\bibitem{long2010random}
Philip~M Long and Rocco~A Servedio.
\newblock Random classification noise defeats all convex potential boosters.
\newblock {\em Machine Learning}, 78(3):287--304, 2010.

\bibitem{freund1990boosting}
Yoav Freund.
\newblock Boosting a weak learning algorithm by majority.
\newblock {\em Information and Computation}, 121(2):256--285, 1995.

\bibitem{freund2001adaptive}
Yoav Freund.
\newblock An adaptive version of the boost by majority algorithm.
\newblock {\em Machine Learning}, 43(3):293--318, 2001.

\bibitem{freund2009more}
Yoav Freund.
\newblock A more robust boosting algorithm.
\newblock {\em arXiv preprint arXiv:0905.2138}, 2009.

\bibitem{masnadi2009design}
Hamed Masnadi-Shirazi and Nuno Vasconcelos.
\newblock On the design of loss functions for classification: theory,
  robustness to outliers, and savageboost.
\newblock In {\em Advances in Neural Information Processing Systems}, pages
  1049--1056, 2009.

\bibitem{miao2016rboost}
Qiguang Miao, Ying Cao, Ge~Xia, Maoguo Gong, Jiachen Liu, and Jianfeng Song.
\newblock Rboost: label noise-robust boosting algorithm based on a nonconvex
  loss function and the numerically stable base learners.
\newblock {\em IEEE Transactions on Neural Networks and Learning Systems},
  27(11):2216--2228, 2016.

\bibitem{bengio2009curriculum}
Yoshua Bengio, J{\'e}r{\^o}me Louradour, Ronan Collobert, and Jason Weston.
\newblock Curriculum learning.
\newblock In {\em Proceedings of the 26th Annual International Conference on
  Machine Learning}, pages 41--48. ACM, 2009.

\bibitem{kumar2010self}
M~Pawan Kumar, Benjamin Packer, and Daphne Koller.
\newblock Self-paced learning for latent variable models.
\newblock In {\em Advances in Neural Information Processing Systems}, pages
  1189--1197, 2010.

\bibitem{jiang2014easy}
Lu~Jiang, Deyu Meng, Teruko Mitamura, and Alexander~G Hauptmann.
\newblock Easy samples first: self-paced reranking for zero-example multimedia
  search.
\newblock In {\em Proceedings of the 22nd ACM International Conference on
  Multimedia}, pages 547--556. ACM, 2014.

\bibitem{zhao2015self}
Qian Zhao, Deyu Meng, Lu~Jiang, Qi~Xie, Zongben Xu, and Alexander~G Hauptmann.
\newblock Self-paced learning for matrix factorization.
\newblock In {\em Proceedings of the Twenty-ninth AAAI Conference on Artificial
  Intelligence}, pages 3196--3202. AAAI Press, 2015.

\bibitem{li2016multi}
Hao Li, Maoguo Gong, Deyu Meng, and Qiguang Miao.
\newblock Multi-objective self-paced learning.
\newblock In {\em Proceedings of the Thirtieth AAAI Conference on Artificial
  Intelligence}, pages 1802--1808. AAAI Press, 2016.

\bibitem{jiang2014self}
Lu~Jiang, Deyu Meng, Shoou-I Yu, Zhenzhong Lan, Shiguang Shan, and Alexander
  Hauptmann.
\newblock Self-paced learning with diversity.
\newblock In {\em Advances in Neural Information Processing Systems}, pages
  2078--2086, 2014.

\bibitem{jiang2015self}
Lu~Jiang, Deyu Meng, Qian Zhao, Shiguang Shan, and Alexander~G Hauptmann.
\newblock Self-paced curriculum learning.
\newblock In {\em Proceedings of the Twenty-ninth AAAI Conference on Artificial
  Intelligence}, pages 2694--2700. AAAI Press, 2015.

\bibitem{zhang2016co}
Dingwen Zhang, Deyu Meng, and Junwei Han.
\newblock Co-saliency detection via a self-paced multiple-instance learning
  framework.
\newblock {\em IEEE Transactions on Pattern Analysis and Machine Intelligence},
  39(5):865--878, 2017.

\bibitem{tang2012shifting}
Kevin Tang, Vignesh Ramanathan, Li~Fei-Fei, and Daphne Koller.
\newblock Shifting weights: adapting object detectors from image to video.
\newblock In {\em Advances in Neural Information Processing Systems}, pages
  638--646, 2012.

\bibitem{supancic2013self}
James~S Supancic and Deva Ramanan.
\newblock Self-paced learning for long-term tracking.
\newblock In {\em Proceedings of the IEEE Conference on Computer Vision and
  Pattern Recognition}, pages 2379--2386. IEEE, 2013.

\bibitem{lee2011learning}
Yong~Jae Lee and Kristen Grauman.
\newblock Learning the easy things first: Self-paced visual category discovery.
\newblock In {\em Proceeding of the IEEE Conference on Computer Vision and
  Pattern Recognition}, pages 1721--1728. IEEE, 2011.

\bibitem{lin2017active}
Liang Lin, Keze Wang, Deyu Meng, Wangmeng Zuo, and Lei Zhang.
\newblock Active self-paced learning for cost-effective and progressive face
  identification.
\newblock {\em To appear in IEEE Transactions on Pattern Analysis and Machine
  Intelligence}, 2017.

\bibitem{kumar2011learning}
M~Pawan Kumar, Haithem Turki, Dan Preston, and Daphne Koller.
\newblock Learning specific-class segmentation from diverse data.
\newblock In {\em 2011 International Conference on Computer Vision}, pages
  1800--1807. IEEE, 2011.

\bibitem{meng2015objective}
Deyu Meng and Qian Zhao.
\newblock What objective does self-paced learning indeed optimize?
\newblock {\em arXiv preprint arXiv:1511.06049}, 2015.

\end{thebibliography}

%\appendices

%\section{}

%\section*{Acknowledgment}


\ifCLASSOPTIONcaptionsoff
  \newpage
\fi

\end{document}